%% file: main_version.tex
\documentclass[11pt,twoside]{article}

\usepackage{xcolor}

\usepackage{fullpage}

\usepackage{epsf}
\usepackage{fancyhdr}
\usepackage{graphics}
\usepackage{graphicx} 
\usepackage{float} 
\usepackage{subfigure} 
\usepackage{psfrag}
\usepackage{comment}

\usepackage[linesnumbered,ruled]{algorithm2e}
\DontPrintSemicolon	

\usepackage{color}
\usepackage{amsthm}
\usepackage{amsfonts}
\usepackage{amsmath}
\usepackage{bm}
\usepackage{amssymb,bbm}
\usepackage[numbers]{natbib}
\usepackage{algorithmic}
\usepackage[usestackEOL]{stackengine}

\usepackage{url}
\usepackage[colorlinks=True,linkcolor=magenta,citecolor=blue,urlcolor=blue,pagebackref=true,backref=true]
{hyperref}
\renewcommand*{\backref}[1]{\ifx#1\relax \else Page #1 \fi}
\renewcommand*{\backrefalt}[4]{%
    \ifcase #1 \footnotesize{(Not cited.)}%
    \or        \footnotesize{(Cited on page~#2.)}%
    \else      \footnotesize{(Cited on pages~#2.)}%
    \fi}

\usepackage{nicefrac}

\usepackage{chngpage}

\usepackage{tabularx}%

\usepackage{enumitem}
\usepackage{booktabs}
\usepackage{pbox}

\usepackage{caption}

\usepackage{mathtools}

\usepackage{fullpage}
\allowdisplaybreaks
\input{final_macros.tex}

\SetKwInput{KwInput}{Input}                
\SetKwInput{KwOutput}{Output}              

\begin{document}
\begin{center}

{\bf{\LARGE{Improving Computational Complexity in Statistical Models
    with Second-Order Information}}}
  
\vspace*{.2in}
{\large{
\begin{tabular}{cccc}
Tongzheng Ren$^{\diamond, \ddag}$ & Jiacheng Zhuo$^{\diamond}$ & Sujay Sanghavi$^{\dagger}$ & Nhat Ho$^{\flat, \ddag}$ \\
\end{tabular}
}}

\vspace*{.1in}

\begin{tabular}{c}
Department of Computer Science, University of Texas at Austin$^\diamond$, \\
Department of Statistics and Data Sciences, University of Texas at Austin$^\flat$ \\
Department of Electrical and Computer Engineering, University of Texas at Austin$^\dagger$, \\
\end{tabular}

\today

\vspace*{.2in}

\begin{abstract}
It is known that when the statistical models are singular, i.e., the Fisher information matrix at the true parameter is degenerate, the fixed step-size gradient descent algorithm takes polynomial number of steps in terms of the sample size $n$ to converge to a final statistical radius around the true parameter, which can be unsatisfactory for the application. To further improve that computational complexity, we consider the utilization of the second-order information in the design of optimization algorithms. Specifically, we study the normalized gradient descent (NormGD) algorithm for solving parameter estimation in parametric statistical models, which is a variant of gradient descent algorithm whose step size is scaled by the maximum eigenvalue of the Hessian matrix of the empirical loss function of statistical models.
    When the population loss function, i.e., the limit of the empirical loss function when $n$ goes to infinity, is homogeneous in all directions, we demonstrate that the NormGD iterates reach a final statistical radius around the true parameter after a logarithmic number of iterations in terms of $n$. Therefore, for fixed dimension $d$, the NormGD algorithm achieves the optimal overall computational complexity $\mathcal{O}(n)$ to reach the final statistical radius. This computational complexity is cheaper than that of the fixed step-size gradient descent algorithm, which is of the order $\mathcal{O}(n^{\tau})$ for some $\tau > 1$, to reach the same statistical radius. We illustrate our general theory under two statistical models: generalized linear models and mixture models, and experimental results support our prediction with general theory.
\end{abstract}
\end{center}
\let\thefootnote\relax\footnotetext{$\ddag$ Correspondence to: Tongzheng Ren (\href{mailto:tongzheng@utexas.edu}{tongzheng@utexas.edu}) and Nhat Ho (\href{mailto:minhnhat@utexas.edu}{minhnhat@utexas.edu}).}
\section{Introduction}
Gradient descent (GD) algorithm has been one of the most well-known and broadly used (first-order) optimization methods for approximating the true parameter for parametric statistical models~\citep{Polyak_Introduction, bubeck2015convex, Nesterov_Introduction}. In unconstrained parameter settings, it is used to solve for optimal solutions $\widehat{\theta}_{n}$ of the following \emph{sample loss} function:
\begin{align}
    \min_{\theta \in \mathbb{R}^{d}} f_{n}(\theta), \label{eq:sample_unconstrained}
\end{align}
where $n$ is the sample size of i.i.d. data $X_{1}, X_{2}, \ldots, X_{n}$ generated from the underlying distribution $P_{\theta^{*}}$. Here, $\theta^{*}$ is the true but unknown parameter. 

When the step size of the gradient descent algorithm is fixed, which we refer to as \emph{fixed-step size gradient descent}, the behaviors of GD iterates for solving the empirical loss function $f_{n}$ can be analyzed via defining the corresponding \emph{population loss} function
\begin{align}
\min_{\theta \in \mathbb{R}^{d}} f(\theta), \label{eq:population_unconstrained}
\end{align}
where $f(\theta) : = \mathbb{E}[f_{n}(\theta)]$ and the outer expectation is taken with respect to the i.i.d. data $X_{1}, X_{2}, \ldots, X_{n}$. An important insight here is that the statistical and computational complexities of fixed-step size sample GD iterates $\theta_{n,\text{GD}}^{t}$ are determined by the singularity of Hessian matrix of the population loss function $f$ at $\theta^{*}$. In particular, when the Hessian matrix of $f$ at $\theta^{*}$ is non-singular, i.e., $\nabla^2 f(\theta^{*}) \succ 0$, the previous works~\cite{Siva_2017, Ho_Instability} demonstrate that $\theta_{n, \text{GD}}^t$ converge to a neighborhood of the true parameter $\theta^{*}$ with the optimal statistical radius $\mathcal{O}((d/n)^{1/2})$ after $\mathcal{O}(\log(n/d))$ number of iterations. The logarithmic number of iterations is a direct consequence of the linear convergence of fixed step size GD algorithm for solving the strongly convex population loss function~\eqref{eq:population_unconstrained}. When the Hessian matrix of $f$ at $\theta^{*}$ is 
singular, i.e., $\text{det}(\nabla^2 f(\theta^{*})) = 0$, which we refer to as \emph{singular statistical models}, $\theta_{n,\text{GD}}^{t}$ can only converge to a neighborhood of $\theta^*$ with the statistical radius larger than $\mathcal{O}((d/n)^{1/2})$ and the iteration complexity becomes polynomial in $n$. In particular, the work of~\cite{Ho_Instability} demonstrates that when the optimization rate of fixed-step size population GD iterates for solving population loss function~\eqref{eq:population_unconstrained} is at the order of $1/ t^{1/\alpha^\prime}$ for some $\alpha^\prime > 0$, 
and the noise magnitude between $\nabla f_{n}(\theta)$ and $\nabla f(\theta)$ is at the order of $\mathcal{O}(r^{\gamma^\prime}(d/n)^{1/2})$ for some $\alpha^\prime \geq \gamma^\prime$ as long as $\|\theta - \theta^{*}\| \leq r$, then the statistical rate of fixed-step size sample GD iterates $\|\theta_{n, \text{GD}}^{t} - \theta^{*}\|$ is $\mathcal{O}((d/n)^{\frac{1}{2(\alpha^\prime + 1 - \gamma^\prime)}})$ after $\mathcal{O}((n/d)^{\frac{\alpha^\prime}{2(\alpha^\prime + 1 - \gamma^\prime)}})$ number of iterations. Given that the per iteration cost of fixed-step size GD is $\mathcal{O}(nd)$, the total computational complexity of fixed-step size GD for solving singular statistical models is $\mathcal{O}(n^{1 + \frac{\alpha^\prime}{2(\alpha^\prime + 1 - \gamma^\prime)}})$ for fixed dimension $d$, which is much more expensive than the optimal computational complexity $\mathcal{O}(n)$. 

\vspace{0.5 em}
\noindent
\textbf{Contribution.} In this paper, to improve the computational complexity of the fixed-step size GD algorithm, we consider the utilization of the second-order information in the design of optimization algorithms. In particular, we study the statistical guarantee of \emph{normalized gradient descent (NormGD) algorithm}, which is a variant of gradient descent algorithm whose step size is scaled by the maximum eigenvalue of the Hessian matrix of the sample loss function, for solving parameter estimation in parametric statistical models. We demonstrate that we are able to obtain the optimal computational complexity $\mathcal{O}(n)$ for fixed dimension $d$ under several settings of (singular) statistical models. Our results can be summarized as follows:
\begin{enumerate}
\item \textbf{General theory:} We study the computational and statistical complexities of NormGD iterates when the population loss function is homogeneous in all directions and the stability of first-order and second-order information holds. In particular, when the population loss function $f$ is homogeneous with all fast directions, i.e., it is locally strongly convex and smooth, and the concentration bounds between the gradients and Hessian matrices of the sample and population loss functions are at the order of $\mathcal{O}((d/n)^{1/2})$, then the NormGD iterates reach the final statistical radius $\mathcal{O}((d/n)^{1/2})$ after $\log(n)$ number of iterations. When the function $f$ is homogeneous, which corresponds to singular statistical models, with the fastest and slowest directions are at the order of $\|\theta - \theta^{*}\|^{\alpha}$ for some $\alpha > 0$, and the concentration bound between Hessian matrices of the sample and population loss functions is $\mathcal{O}(r^{\gamma} (d/n)^{1/2})$ for some $\gamma \geq 0$ and $\alpha \geq \gamma  + 1$, then the NormGD iterates converge to a radius $\mathcal{O}((d/n)^{\frac{1}{2(\alpha - \gamma)}}$ within the true parameter after $\log(n)$ number of iterations. Therefore, for fixed dimension $d$ the total computational complexity of NormGD to reach the final statistical radius is at the order of $\mathcal{O}(n \log(n))$, which is cheaper than that of the fixed step size GD, which is of the order of $\mathcal{O}(n^{1 + \frac{\alpha}{2(\alpha - \gamma)}})$. Details of these results are in Theorem~\ref{theorem:statistical_rate_homogeneous_settings} and Proposition~\ref{proposition:fixed_step_GD}.

\item \textbf{Examples:} We illustrate the general theory for the statistical guarantee of NormGD under two popular statistical models: generalized linear models (GLM) and Gaussian mixture models (GMM). For GLM, we consider the settings when the link function $g(r) = r^{p}$ for $p \in \mathbb{N}$ and $p \geq 2$. We demonstrate that for the strong signal-to-noise regime, namely, when the norm of the true parameter is sufficiently large, the NormGD iterates reach the statistical radius $\mathcal{O}((d/n)^{1/2})$ around the true parameter after $\log(n)$ number of iterations. On the other hand, for the low signal-to-noise regime of these generalized linear models, specifically, we assume the true parameter to be 0, the statistical radius of NormGD updates is $\mathcal{O}((d/n)^{1/2p})$ and it is achieved after $\log(n)$ number of iterations. Moving to the GMM, we specifically consider the symmetric two-component location setting, which has been considered widely to study the statistical behaviors of Expectation-Maximization (EM) algorithm~\cite{Siva_2017, Raaz_Ho_Koulik_2020}. We demonstrate that the statistical radius of NormGD iterates under strong and low signal-to-noise regimes are respectively $\mathcal{O}((d/n)^{1/2})$ and $\mathcal{O}((d/n)^{1/4})$. Both of these results are obtained after $\log(n)$ number of iterations.
\end{enumerate}
To the best of our knowledge, our results of NormGD in the paper are the first attempt to leverage second-order information to improve the computational complexity of optimization algorithms for solving parameter estimation in statistical models. Furthermore, we wish to remark that there are potentially more efficient algorithms than NormGD by employing more structures of the Hessian matrix, such as using the trace of the Hessian matrix as the scaling factor of the GD algorithm. We leave a detailed development for such direction in future work. 

\vspace{0.5 em}
\noindent
\textbf{Related works.} Recently, Ren et al.~\cite{Tongzheng_2022} proposed using Polyak step size GD algorithm to obtain the optimal computational complexity $\mathcal{O}(n)$ for reaching the final statistical radius in statistical models. They demonstrated that for locally strongly convex and smooth population loss functions, the Polyak step size GD iterates reach the similar statistical radius $\mathcal{O}((d/n)^{1/2})$ as that of the fixed-step size GD with similar iteration complexity $\log(n)$. For the singular statistical settings, when the population loss function satisfies the generalized smoothness and generalized Łojasiewicz property, which are characterized by some constant $\alpha^\prime > 0$, and the deviation bound between the gradients of sample and population loss functions is $\mathcal{O}(r^{\gamma^\prime}(d/n)^{1/2})$ for some $\alpha^\prime \geq \gamma^\prime$, then the statistical rate of Polyak step size GD iterates is $\mathcal{O}((d/n)^{\frac{1}{2(\alpha^\prime + 1 - \gamma^\prime)}})$ after $\mathcal{O}(\log(n))$ number of iterations. Therefore, for fixed dimension $d$, the total computational complexity of Polyak step size GD algorithm for reaching the final statistical radius is $\mathcal{O}(n)$. Even though this complexity is comparable to that of NormGD algorithm, the Polyak step size GD algorithm requires the knowledge of the optimal value of the sample loss function, i.e., $\min_{\theta \in \mathbb{R}^{d}} f_{n}(\theta)$, which is not always simple to estimate. 

\vspace{0.5 em}
\noindent
\textbf{Organization.} The paper is organized as follows. In Section~\ref{sec:general_theory_normalized_GD} and Appendix~\ref{sec:homogenenous_fast}, we provide a general theory for the statistical guarantee of the NormGD algorithm for solving parameter estimation in parametric statistical models when the population loss function is homogeneous. We illustrate the general theory with generalized linear models and mixture models in Section~\ref{sec:examples}. We conclude the paper with a few discussions in Section~\ref{sec:discussion}. Finally, proofs of the general theory are in Appendix~\ref{sec:proof_main} while proofs of the examples are in the remaining appendices in the supplementary material.

\vspace{0.5 em}
\noindent
\textbf{Notation.} For any $n \in \mathbb{N}$, we denote $[n] = \{1, 2, \ldots, n\}$. For any matrix $A \in \mathbb{R}^{d \times d}$, we denote $\lambda_{\max}(A)$, $\lambda_{\min}(A)$ respectively the maximum and minimum eigenvalues of the matrix $A$. Throughout the paper, $\|\cdot\|$ denotes the $\ell_{2}$ norm of some vector while $\|\cdot\|_{\text{op}}$ denotes the operator norm of some matrix. For any two sequences $\{a_{n}\}_{n \geq 1}, \{b_{n}\}_{n \geq 1}$, the notation $a_{n} = \mathcal{O}(b_{n})$ is equivalent to $a_{n} \leq C b_{n}$ for all $n \geq 1$ where $C$ is some universal constant.

\section{General Theory of Normalized Gradient Descent}
\label{sec:general_theory_normalized_GD}
In this section, we provide statistical and computational complexities of NormGD updates for homogeneous settings when all the directions of the population loss function $f$ have similar behaviors. For the inhomogeneous population loss function, to the best of our knowledge, the theories for these settings are only for specific statistical models~\cite{Raaz_Ho_Koulik_2018_second, Zhuo_2021}. The general theory for these settings is challenging and hence we leave this direction for future work. To simplify the ensuing presentation, we denote the NormGD iterates for solving the samples and population losses functions~\eqref{eq:sample_unconstrained} and~\eqref{eq:population_unconstrained} as follows:
\begin{align*}
    \theta_{n}^{t + 1} & : = F_{n}(\theta_{n}^{t}) = \theta_{n}^{t} - \frac{\eta}{\lambda_{\max}(\nabla^2 f_{n}(\theta_{n}^{t}))} \nabla f_{n}(\theta_{n}^{t}), \\
    \theta^{t + 1} & : = F(\theta^{t}) = \theta^{t} - \frac{\eta}{\lambda_{\max}(\nabla^2 f(\theta^{t}))} \nabla f(\theta^{t}).
\end{align*}
where $F_{n}$ and $F$ are the sample and population NormGD operators. Furthermore, we call $\theta_{n}^{t}$ and $\theta^{t}$ as the sample and population NormGD iterates respectively. 

For the homogeneous setting when all directions are fast, namely, when the population loss function is locally strongly convex, we defer the general theory of these settings to Appendix~\ref{sec:homogenenous_fast}. Here, we only consider the homogeneous settings where all directions are slow. To characterize the homogeneous settings, we assume that the population loss function $f$ is locally convex in $\mathbb{B}(\theta^{*}, r)$ for some given radius $r$. Apart from the local convexity assumption, we also utilize the following assumption on the population loss function $f$.
\begin{enumerate}[label=(W.1)] 
\item \label{assump:homogeneous_property} (Homogeneous Property) 
Given the constant $\alpha > 0$ and the radius $r > 0$, for all $\theta\in \mathbb{B}(\theta^*, r)$ we have
\begin{align*}
    \lambda_{\min}(\nabla^2 f(\theta)) \geq & c_1 \|\theta - \theta^*\|^{\alpha}, \\
    \lambda_{\max}(\nabla^2 f(\theta)) \leq & c_2 \|\theta - \theta^*\|^{\alpha},
\end{align*}
where $c_{1} > 0$ and $c_2>0$ are some universal constants depending on $r$.
\end{enumerate}
The condition $\alpha > 0$ is to ensure that the Hessian matrix is singular at the true parameter $\theta^{*}$. For the setting $\alpha = 0$, corresponding to the locally strongly convex setting, the analysis of NormGD is in Appendix~\ref{sec:homogenenous_fast}. A simple example of Assumption~\ref{assump:homogeneous_property} is $f(\theta) = \|\theta - \theta^{*}\|^{\alpha + 2}$ for all $\theta \in \mathbb{B}(\theta^{*}, r)$. The Assumption~\ref{assump:homogeneous_property} is satisfied by several statistical models, such as low signal-to-noise regime of generalized linear models with polynomial link functions (see Section~\ref{sec:example_glm}) and symmetric two-component mixture model when the the true parameter is close to 0 (see Section~\ref{sec:Gaussian_mixture}). The homogeneous assumption~\ref{assump:homogeneous_property} was also considered before to study the statistical and computational complexities of optimization algorithms~\cite{Tongzheng_2022}.

\vspace{0.5 em}
\noindent
\textbf{Statistical rate of sample NormGD iterates $\theta_{n}^{t}$:} To establish the statistical and computational complexities of sample NormGD updates $\theta_{n}^{t}$, we utilize the population to sample analysis~\cite{yi2015regularized, Siva_2017, Ho_Instability, Kwon_minimax}. In particular, an application of triangle inequality leads to
\begin{align}
    \|\theta_{n}^{t + 1} - \theta^{*}\| & \leq \|F_{n}(\theta_{n}^{t}) - F(\theta_{n}^{t})\| + \|F(\theta_{n}^{t}) - \theta^{*}\| \nonumber \\
    & = A + B. \label{eq:population_to_sample}
\end{align}
Therefore, the statistical radius of $\theta_{n}^{t + 1}$ around $\theta^{*}$ is controlled by two terms: (1) Term A: the uniform concentration of the sample NormGD operator $F_{n}$ around the population GD operator $F$; (2) Term B: the contraction rate of population NormGD operator.

For term B in equation~\eqref{eq:population_to_sample}, the homogeneous assumption~\ref{assump:homogeneous_property} entails the following contraction rate of population NormGD operator.
\begin{lemma} \label{lemma:optimization_homogeneous}
Assume Assumption~\ref{assump:homogeneous_property} holds for some $\alpha > 0$ and some universal constants $c_1, c_2$. Then, if the step-size $\eta \leq \frac{c_1^2}{2c_2^2}$, then we have that
\begin{align*}
    \|F(\theta) - \theta^*\| \leq \kappa\|\theta - \theta^*\|,
\end{align*}
where $\kappa < 1$ is a universal constant that only depends on $\eta, c_1, c_2, \alpha$.
\end{lemma}
The proof of Lemma~\ref{lemma:optimization_homogeneous} is in Appendix~\ref{sec:proof:theorem:optimization_homogeneous}. For term A in equation~\eqref{eq:population_to_sample}, the uniform concentration bound between $F_{n}$ and $F$, it can be obtained via the following assumption on the concentration bound of the operator norm of $\nabla^2 f_{n}(\theta) - \nabla^2 f(\theta)$ as long as $\|\theta - \theta^{*}\| \leq r$.
\begin{enumerate}[label=(W.2)]
\item\label{assump:concentration_bound} (Stability of Second-order Information) For a given parameter $\gamma \geq 0$, there exist a noise function $\varepsilon: \mathbb{N} \times (0,1] \to \mathbb{R}^{+}$, universal constant $c_3 > 0$, and some positive parameter $\rho > 0$ such that 
\begin{align*}
    \sup_{\theta\in \mathbb{B}(\theta^*, r)} \|\nabla^{2} f_{n}(\theta) - \nabla^2 f(\theta)\|_{\text{op}} \leq c_3 r^\gamma \varepsilon(n, \delta),
\end{align*}
for all $r \in (0, \rho)$ with probability $1 - \delta$.
\end{enumerate}
To the best of our knowledge, the stability of second-order information in Assumption~\ref{assump:concentration_bound} is novel and has not been considered before to analyze the statistical guarantee of optimization algorithms. The idea of Assumption~\ref{assump:concentration_bound} is to control the growth of noise function, which is the difference between the population and sample loss functions, via the second-order information of these loss functions. A simple example for Assumption~\ref{assump:concentration_bound} is when $f_{n}(\theta) = \frac{\|\theta\|^{2p}}{2p} - \omega \frac{\|\theta\|^{2q}}{2q} \sqrt{\frac{d}{n}}$ where $\omega \sim \mathcal{N}(0, 1)$ and $p, q$ are some positive integer numbers such that $p > q$. Then, $f(\theta) = \|\theta\|^{2p}/ 2p$. The Assumption~\ref{assump:concentration_bound} is satisfied with $\gamma = 2q - 2$ and with the noise function $\varepsilon(n, \delta) = \sqrt{\frac{d \log(1/\delta)}{n}}$. For concrete statistical examples, we demonstrate later in Section~\ref{sec:examples} that Assumption~\ref{assump:concentration_bound} is satisfied by generalized linear model and mixture model. 

Given Assumption~\ref{assump:concentration_bound}, we have the following uniform concentration bound between the sample NormGD operator $F_{n}$ and population NormGD operator $F$.
\begin{lemma} \label{lemma:concentration_bound_population_sample}
Assume that Assumptions~\ref{assump:homogeneous_property} and~\ref{assump:concentration_bound} hold with $\alpha \geq \gamma + 1$. Furthermore, assume that $\nabla f_{n}(\theta^{*}) = 0$. Then, we obtain that
\begin{align*}
    \sup_{\theta \in \mathbb{B}(\theta^{*}, r) \backslash \mathbb{B}(\theta^{*}, r_{n})} \|F_n(\theta)-F(\theta)\| \leq c_4 r^{\gamma + 1 - \alpha} \varepsilon(n, \delta),
\end{align*}
where $r_n : = \left(\frac{6c_3 \varepsilon(n, \delta)}{c_1}\right)^{\frac{1}{\alpha - \gamma}}$, and $c_4$ is a universal constant depends on $\eta, c_1, c_2, c_3, \alpha, \gamma$.
\end{lemma}
The proof of Lemma~\ref{lemma:concentration_bound_population_sample} is in Appendix~\ref{sec:proof:lemma:concentration_bound_population_sample}. We have a few remarks with Lemma~\ref{lemma:concentration_bound_population_sample}. First, the assumption that $\nabla f_{n}(\theta^{*}) = 0$ is to guarantee the stability of $\nabla f_{n}(\theta)$ around $\nabla f(\theta)$ as long as $\|\theta - \theta^{*}\| \leq r$ for any $r > 0$~\cite{Ho_Instability}. This assumption is satisfied by several models, such as low signal-to-noise regimes of generalized linear model and Gaussian mixture models in Section~\ref{sec:examples}. Second, the assumption that $\alpha \geq \gamma + 1$ means that the signal is stronger than the noise in statistical models, which in turn leads to meaningful statistical rates. Third, the inner radius $r_{n}$ in Lemma~\ref{lemma:concentration_bound_population_sample} corresponds to the final statistical radius, which is at the order $\mathcal{O}(\varepsilon(n, \delta)^{\frac{1}{\alpha - \gamma}})$. It means that we cannot go beyond that radius, or otherwise the empirical Hessian is not positive definite. 

Based on the contraction rate of population NormGD operator in Lemma~\ref{lemma:optimization_homogeneous} and the uniform concentration of the sample NormGD operator around the population NormGD operator in Lemma~\ref{lemma:concentration_bound_population_sample}, we have the following result on the statistical and computational complexities of the sample NormGD iterates around the true parameter $\theta^{*}$.
\begin{theorem} \label{theorem:statistical_rate_homogeneous_settings}
Assume that Assumptions~\ref{assump:homogeneous_property} and~\ref{assump:concentration_bound} and assumptions in Lemma~\ref{lemma:concentration_bound_population_sample} hold with $\alpha \geq \gamma + 1$. Assume that the sample size $n$ is large enough such that $\varepsilon(n, \delta)^{\frac{1}{\alpha - \gamma}} \leq \frac{(1 - \kappa) r}{c_{4} \bar{C}^{\gamma + 1 - \alpha}}$ where $\kappa$ is defined in Lemma~\ref{lemma:optimization_homogeneous}, $c_{4}$ is the universal constant in Lemma~\ref{lemma:concentration_bound_population_sample} and $\bar{C} = (\frac{6c_3}{c_1})^{\frac{1}{\alpha -\gamma}}$, and $r$ is the local radius. Then, there exist universal constants $C_{1}$, $C_{2}$ such that with probability $1 - \delta$, for $t \geq C_{1}\log (1/\varepsilon(n, \delta))$, the following holds:
\begin{align*}
    \min_{k\in \{0, 1, \cdots, t\}}\|\theta_n^k - \theta^*\| \leq C_{2} \cdot \varepsilon(n, \delta)^{\frac{1}{\alpha  - \gamma}}.
\end{align*}
\end{theorem}
The proof of Theorem~\ref{theorem:statistical_rate_homogeneous_settings} follows the argument of part (b) of Theorem 2 in~\cite{Ho_Instability}; therefore, it is omitted. A few comments with Theorem~\ref{theorem:statistical_rate_homogeneous_settings} are in order. 

\vspace{0.5 em}
\noindent
\textbf{On the approximation of $\lambda_{\max}$:} Computing the whole spectrum of a $d\times d$ matrix requires $\mathcal{O}(d^3)$ computation. But fortunately, we can compute the maximum eigenvalue in $\mathcal{O}(d^2)$ computation with the well-known power iteration~\citep[\emph{a.k.a} power method, see Chapter 7.3,][]{golub1996matrix}) which has broad applications in different areas \citep[e.g.][]{hardt16}. Power iteration can compute the maximum eigenvalue up to $\varepsilon$ error with at most $\mathcal{O}\left(\frac{\log \varepsilon}{\log(\lambda_2 / \lambda_{\max})}\right)$ matrix vector products, where $\lambda_2$ is the second largest eigenvalue. Hence, when $\lambda_2/\lambda_{\max}$ is bounded away from $1$, we can obtain a high-quality approximation of $\lambda_{\max}$ with small number of computation. Things can be a little weird when $\lambda_2/\lambda_{\max}$ is close to $1$. But in fact, we only requires an approximation of $\lambda_{\max}$ within statistical accuracy defined in Assumption \ref{assump:concentration_bound}. Hence, without loss of generality, we can assume $\lambda_2(\nabla^2 f_n(\theta))\leq \lambda_{\max}(\nabla^2 f_n(\theta)) - c_3 \|\theta - \theta^*\|^\gamma \varepsilon(n, \delta)$, which means $\frac{\lambda_2(\nabla^2 f_n(\theta))}{ \lambda_{\max}(\nabla^2 f_n(\theta))} \leq 1 - \frac{c_3}{c_2}\|\theta - \theta^*\|^{\gamma - \alpha}$. Since $\alpha \geq \gamma + 1$ and we only consider the case $\|\theta - \theta^*\| \leq r$, we know their exists a universal constant $c_{\mathrm{PI}} < 1$ that does not depend on $n, d$, such that $\lambda_2 / \lambda_{\max} \leq c_{\mathrm{PI}}$. As a result, we can always compute the $\lambda_{\max}$ with small number of iterations.

\vspace{0.5 em}
\noindent
\textbf{Comparing to fixed-step size gradient descent:} Under the Assumptions~\ref{assump:homogeneous_property} and~\ref{assump:concentration_bound}, we have the following result regarding the statistical and computational complexities of fixed-step size GD iterates.
\begin{proposition}
\label{proposition:fixed_step_GD}
Assume that Assumptions~\ref{assump:homogeneous_property} and~\ref{assump:concentration_bound} hold with $\alpha \geq \gamma + 1$ and $\nabla f_{n}(\theta^{*}) = 0$. Suppose the sample size $n$ is large enough so that $\varepsilon(n, \delta) \leq C$ for some universal constant $C$. Then there exist universal constant $C_1$ and $C_2$, such that for any fixed $\tau \in \left(0, \frac{1}{\alpha - \gamma}\right)$, as long as $t \geq C_1 \varepsilon(n, \delta)^{-\frac{\alpha}{\alpha - \gamma}} \log \frac{1}{\tau}$, we have that
\begin{align*}
    \|\theta_{n, \text{GD}}^t - \theta^*\| \leq C_2 \varepsilon(n, \delta)^{\frac{1}{\alpha - \gamma} - \tau}.
\end{align*}
\end{proposition}
The proof of Proposition~\ref{proposition:fixed_step_GD} is similar to Proposition 1 in~\cite{Tongzheng_2022}, and we omit the proof here. Therefore, the results in Theorem~\ref{theorem:statistical_rate_homogeneous_settings} indicate that the NormGD and fixed-step size GD iterates reach the same statistical radius $\varepsilon(n, \delta)^{\frac{1}{\alpha - \gamma}}$ within the true parameter $\theta^{*}$. Nevertheless, the NormGD only takes $\mathcal{O}(\log (1/\varepsilon(n, \delta)))$ number of iterations while the fixed-step size GD takes $\mathcal{O}(\varepsilon(n, \delta)^{-\frac{\alpha}{\alpha - \gamma}})$ number of iterations. If the dimension $d$ is fixed, the total computational complexity of NormGD algorithm is at the order of $\mathcal{O}(n \cdot \log (1/\varepsilon(n, \delta)))$, which is much cheaper than that of fixed-step size GD,  $\mathcal{O}(n \cdot \varepsilon(n, \delta)^{-\frac{\alpha}{\alpha - \gamma}})$, to reach the final statistical radius.

\section{Examples}
\label{sec:examples}
In this section, we consider an application of our theories in previous section to the generalized linear model and Gaussian mixture model. 
\subsection{Generalized Linear Model (GLM)}
\label{sec:example_glm}
Generalized linear model (GLM) has been a widely used model in statistics and machine learning~\cite{Nelder_1972}. It is a generalization of linear regression model where we use a link function to relate the covariates to the response variable. In particular, we assume that $(Y_{1}, X_{1}), \ldots, (Y_{n}, X_{n}) \in \mathbb{R} \times \mathbb{R}^{d}$ satisfy
\begin{align}
    Y_{i} = g(X_{i}^{\top}\theta^{*}) + \varepsilon_{i}. \quad \quad \forall i \in [n] \label{eq:generalized_linear}
\end{align}
Here, $g: \mathbb{R} \to \mathbb{R}$ is a given link function, $\theta^{*}$ is a true but unknown parameter, and $\varepsilon_{1},\ldots,\varepsilon_{n}$ are i.i.d. noises from $\mathcal{N}(0, \sigma^2)$ where $\sigma > 0$ is a given variance parameter. We consider the random design setting where $X_{1}, \ldots, X_{n}$ are i.i.d. from $\mathcal{N}(0, I_{d})$. A few comments with our model assumption. First, in our paper, we will not estimate the link function $g$. Second, the assumption that the noise follows the Gaussian distribution is just for the simplicity of calculations; similar proof argument still holds for sub-Gaussian noise. For the purpose of our theory, we consider the link function $g(r) : = r^{p}$ for any $p \in \mathbb{N}$ and $p \geq 2$. When $p = 2$, the generalize linear model becomes the phase retrieval problem~\citep{Fienup_82,Shechtman_Yoav_etal_2015, candes_2011,Netrapalli_Prateek_Sanghavi_2015}.

\vspace{0.5 em}
\noindent
\textbf{Least-square loss:} We estimate the true parameter $\theta^{*}$ via minimizing the least-square loss function, which is:
\begin{align}
    \min_{\theta \in \mathbb{R}^{d}} \mathcal{L}_{n}(\theta) : = \frac{1}{2n} \sum_{i = 1}^{n} (Y_{i} - (X_{i}^{\top} \theta)^{p})^{2}. \label{eq:sample_loss_linear}
\end{align}
By letting the sample size $n$ goes to infinity, we obtain the population least-square loss function of GLM:
\begin{align*}
    \min_{\theta \in \mathbb{R}^{d}} \mathcal{L}(\theta) : = \frac{1}{2} \mathbb{E}_{X, Y}[(Y - (X^{\top} \theta)^{p})^{2}],
\end{align*}
where the outer expectation is taken with respect to $X \sim \mathcal{N}(0, I_{d})$ and $Y = g(X^{\top} \theta^{*}) + \varepsilon$ where $\varepsilon \sim \mathcal{N}(0, \sigma^2)$.
It is clear that $\theta^{*}$ is the global minimum of the population loss function $\mathcal{L}$. Furthermore, the function $\mathcal{L}$ is homogeneous, i.e., all directions have similar behaviors.

In this section, we consider two regimes of the GLM for our study of sample NormGD iterates: Strong signal-to-noise regime and Low signal-to-noise regime.

\vspace{0.5 em}
\noindent
\textbf{Strong signal-to-noise regime:} The strong signal-to-noise regime corresponds to the setting when $\theta^{*}$ is bounded away from 0 and $\|\theta^{*}\|$ is sufficiently large, i.e., $\|\theta^{*}\| \geq C$ for some universal constant $C$. Under this setting, we can check that the population loss function $\mathcal{L}$ is locally strongly convex and smooth, i.e., it satisfies Assumption~\ref{assump:homogeneous_property_fast} under the homogeneous setting with all fast directions. Furthermore, for Assumption~\ref{assump:concentration_bound_fast}, for any radius $r > 0$ there exist universal constants $C_{1}, C_{2}, C_{3}$ such that as long as $n \geq C_{1} (d \log(d/\delta))^{2p}$, the following uniform concentration bounds hold:
\begin{align}
    \sup_{\theta\in \mathbb{B}(\theta^*, r)} \|\nabla \mathcal{L}_{n}(\theta) - \nabla \mathcal{L}(\theta)\| \leq C_{2} \sqrt{\frac{d + \log(1/ \delta)}{n}}, \label{eq:generalized_linear_concentration_first_order} \\
    \sup_{\theta\in \mathbb{B}(\theta^*, r)} \|\nabla^2 \mathcal{L}_{n}(\theta) - \nabla^2 \mathcal{L}(\theta)\|_{\text{op}} \leq C_{3} \sqrt{\frac{d + \log(1/ \delta)}{n}}. \label{eq:generalized_linear_concentration_second_order}
\end{align}
The proof can be found in Appendix~\ref{sec:proof:low_signal_generalized_linear_uniform_concentration}.

\vspace{0.5 em}
\noindent
\textbf{Low signal-to-noise regime:} The low signal-to-noise regime corresponds to the setting when the value of $\|\theta^{*}\|$ is sufficiently small. To simplify the computation, we assume that $\theta^{*} = 0$. Direct calculation shows that $\nabla \mathcal{L}_{n}(\theta^{*}) = 0$. Furthermore, the population loss function becomes
\begin{align}
    \min_{\theta \in \mathbb{R}^{d}} \mathcal{L}(\theta) = \frac{\sigma^2 + (2p - 1)!! \|\theta - \theta^{*}\|^{2p}}{2}. \label{eq:population_no_signal}
\end{align}
Under this setting, the population loss function $\mathcal{L}$ is no longer locally strong convex around $\theta^{*} = 0$. Indeed, this function is homogeneous with all slow directions, which are given by:
\begin{align}
    \lambda_{\max}(\nabla^2 \mathcal{L}(\theta)) & \leq c_{1} \|\theta - \theta^{*}\|^{2p - 2}, \label{eq:low_signal_generalized_linear_maximum_eigenvalue} \\
    \lambda_{\min}(\nabla^2 \mathcal{L}(\theta)) & \geq c_{2} \|\theta - \theta^{*}\|^{2p - 2}, \label{eq:low_signal_generalized_linear_minimum_eigenvalue}
\end{align}
for all $\theta \in \mathbb{B}(\theta^{*}, r)$ for some $r > 0$. Here, $c_{1}, c_{2}$ are some universal constants depending on $r$. Therefore, the homogeneous Assumption~\ref{assump:homogeneous_property} is satisfied with $\alpha = 2p - 2$. The proof for the claims~\eqref{eq:low_signal_generalized_linear_maximum_eigenvalue} and~\eqref{eq:low_signal_generalized_linear_minimum_eigenvalue} is in Appendix~\ref{sec:proof:low_signal_generalized_linear_homogeneous}.

Moving to Assumption~\ref{assump:concentration_bound}, we demonstrate in Appendix~\ref{sec:proof:low_signal_generalized_linear_uniform_concentration} that we can find universal constants $C_{1}$ and $C_{2}$ such that for any $r > 0$ and $n \geq C_{1} (d \log(d/ \delta))^{2p}$ we have
\begin{align}
    \sup_{\theta \in \mathbb{B}(\theta^{*}, r)} \|
    \nabla^2 \mathcal{L}_{n}(\theta) - \nabla^2 \mathcal{L}(\theta) \|_{\text{op}} \leq C_{2} (r^{p - 2} + r^{2p - 2}) \sqrt{\frac{d + \log(1/ \delta)}{n}}  \label{eq:concentration_Hessian_generalized_model}
\end{align}
with probability at least $1 - \delta$. Hence, the stability of second order information Assumption~\ref{assump:concentration_bound} is satisfied with $\gamma = p - 2$.

Based on the above results, Theorems~\ref{theorem:statistical_rate_homogeneous_settings} for homogeneous settings with all slow directions and~\ref{theorem:statistical_rate_homogeneous_settings_fast} for homogeneous settings with all fast directions lead to the following statistical and computational complexities of NormGD algorithm for solving the true parameter of GLM.

\begin{corollary}
\label{corollary:generalized_model} 
Given the generalized linear model~\eqref{eq:generalized_linear} with $g(r) = r^{p}$ for some $p \in \mathbb{N}$ and $p \geq 2$, there exists universal constants $c, \tilde{c}_{1}, \tilde{c}_{2}, \bar{c}_{1}, \bar{c}_{2}$ such that when the sample size $n \geq c (d \log(d/ \delta))^{2p}$ and the initialization $\theta_{n}^{0} \in \mathbb{B}(\theta^{*}, r)$ for some chosen radius $r > 0$, with probability $1 - \delta$ the sequence of sample NormGD iterates $\{\theta_{n}^{t}\}_{t \geq 0}$ satisfies the following bounds:

(i) When $\|\theta^{*}\| \geq C$ for some universal constant $C$, we find that
\begin{align*}
        \| \theta_{n}^{t} - \theta^{*} \|  \leq \tilde{c}_{1} \sqrt{\frac{d + \log(1/\delta)}{n}}, \quad \quad \text{for} \ t \geq \tilde{c}_{2} \log \parenth{\frac{n}{d + \log(1/\delta)}},
\end{align*}
(ii) When $\theta^{*} = 0$, we obtain
\begin{align*}
        \min_{1 \leq k \leq t} \| \theta_{n}^{k} - \theta^{*} \| \leq c_{1}^\prime \parenth{\frac{d + \log(1/\delta)}{n}}^{1/(2p)}, \quad \quad \text{for} \ t \geq c_{2}^\prime \log \parenth{\frac{n}{d + \log(1/\delta)}}.
\end{align*}
\end{corollary}
A few comments with Corollary~\ref{corollary:generalized_model} are in order. For the strong signal-to-noise regime, the sample NormGD only takes logarithmic number of iterations $\log(n)$ to reach the optimal statistical radius $(d/n)^{1/2}$ around the true parameter. This guarantee is similar to that of the fixed-step size GD iterates for solving the locally strongly convex and smooth loss function~\cite{Siva_2017, Ho_Instability}. For the low signal-to-noise regime, the sample NormGD iterates reach the final statistical radius $(d/n)^{1/2p}$ after logarithmic number of iterations in terms of $n$. In terms of the number of iterations, it is cheaper than that of that fixed-step size GD algorithm, which takes at least $\mathcal{O}(n^{\frac{p - 1}{p}})$ number of iterations for fixed dimension $d$ (See our discussion after Theorem~\ref{theorem:statistical_rate_homogeneous_settings}). For fixed $d$, it indicates that the total computational complexity of NormGD algorithm, which is at the order of $\mathcal{O}(n)$, is smaller than that of fixed-step size GD, which is $\mathcal{O}(n^{1 + \frac{p - 1}{p}})$. Therefore, for the low signal-to-noise regime, the NormGD algorithm is more computationally efficient than the fixed-step size GD algorithm for reaching the similar final statistical radius.

\vspace{0.5 em}
\noindent
\textbf{Experiments:}
\begin{figure*}[t!]
  \centering \includegraphics[width=0.32\textwidth]{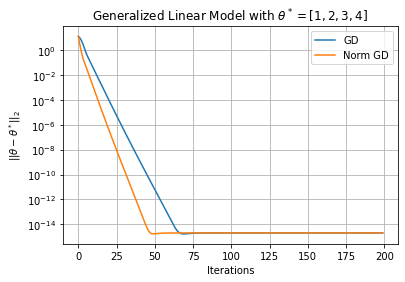}
  \includegraphics[width=0.32\textwidth]{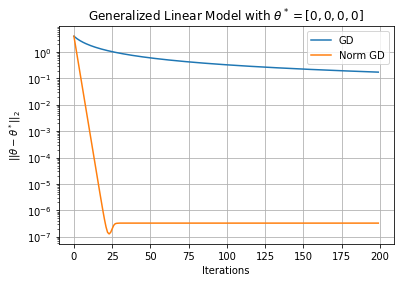}
  \includegraphics[width=0.32\textwidth]{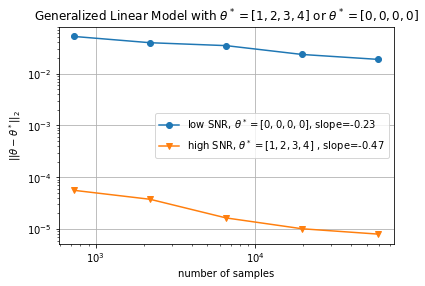}
  \caption{\textit{Verification simulation for the Generalized Linear Model (GLM) example}. \textbf{Left}: Both GD and Norm GD converges linearly in the high signal-to-noise setting; \textbf{Middle}: only Norm GD converges linearly in the low signal-to-noise setting while GD converges sub-linearly; \textbf{Right}: the log-log plot of sample size versus statistical error shows that the statistical error scales with $n^{-0.5}$ in the strong signal-to-noise setting and $n^{-0.25}$ in the low signal-to-noise setting, which coincides with our theory. The slope is computed as the linear regression coefficient of the log sample size versus the log statistical error.}
  \label{fig:GLM-exp}
\end{figure*}
To verify our theory, we performed simulation on generalized linear model, and the results are shown in Figure~\ref{fig:GLM-exp}. We set $p=2$ and $d=4$. For the low signal-to-noise setting, we set $\theta^*$ to be $[0,0,0,0]$, and for high signal-to-noise setting, we set $\theta^*$ to be $[1,2,3,4]$. For the left and the middle plots in Figure~\ref{fig:GLM-exp}, the sample size is set to be $1000$. As in the left plot, when in the strong signal-to-noise setting, both the fixed step size Gradient Descent method (referred to as GD) and our proposed Normalized Gradient Descent method (referred to as NormGD) converges linearly. However, once we shift to the low signal-to-noise setting, only Norm GD converges linearly, while GD converges only sub-linearly, as shown in the middle plot of Figure~\ref{fig:GLM-exp}. To further verify our corollaries, especially how the statistical error scales with $n$, we plot the statistical error versus sample size in the right plot as in Figure~\ref{fig:GLM-exp}. The experiments was repeated for $10$ times and the average of the statistical error is shown. The slope is computed as the linear regression coefficient of the log sample size versus the log statistical error. As in this log-log plot, in the strong signal-to-noise setting, the statistical error roughly scales with $n^{-0.5}$, while in the low signal-to-noise setting, the statistical error roughly scales with $n^{-0.25}$. This coincides with our theory as in Corollary~\ref{corollary:generalized_model}.
\subsection{Gaussian mixture models (GMM)}
\label{sec:Gaussian_mixture}

We now consider Gaussian mixture models (GMM), one of the most popular statistical models for modeling heterogeneous data~\citep{Lindsay-1995, Mclachlan-1988}. Parameter estimation in these models plays an important role in capturing the heterogeneity of different subpopulations. The common approach to estimate the location and scale parameters in these model is via maximizing the log-likelihood function. The statistical guarantee of the maximum likelihood estimator (MLE) in Gaussian mixtures had been studied in~\cite{Chen1992, Ho-Nguyen-AOS-17}. However, since the log-likelihood function is highly non-concave, in general we do not have closed-form expressions for the MLE. Therefore, in practice we utilize optimization algorithms to approximate the MLE. However, a complete picture about the statistical and computational complexities of these optimization algorithms have remained poorly understood.

In order to shed light on the behavior of NormGD algorithm for solving GMM, we consider a simplified yet important setting of this model, symmetric two-component location GMM. This model had been used in the literature to study the statistical behaviors of Expectation-Maximization (EM) algorithm~\cite{Siva_2017, Raaz_Ho_Koulik_2020}. We assume that the data $X_{1}, X_{2}, \ldots, X_{n}$ are i.i.d. samples from $\frac{1}{2} \mathcal{N}(-\theta^{*}, \sigma^2 I_{d}) + \frac{1}{2} \mathcal{N}(\theta^{*}, \sigma^2 I_{d})$ where $\sigma > 0$ is given and $\theta^{*}$ is true but unknown parameter. Our goal is to obtain an estimation of $\theta^{*}$ via also using the symmetric two-component location Gaussian mixture:
\begin{align}
\frac{1}{2} \mathcal{N}(-\theta, \sigma^2 I_{d}) + \frac{1}{2} \mathcal{N}(\theta, \sigma^2 I_{d}). \label{eq:overspecify_mixture}
\end{align}
As we mentioned earlier, we obtain an estimation of $\theta^{*}$ via maximizing the sample log-likelihood function associated with model~\eqref{eq:overspecify_mixture}, which admits the following form:
\begin{align}
    \min_{\theta \in \mathbb{R}^{d}} \bar{\mathcal{L}}_{n}(\theta) : = -\frac{1}{n} \sum_{i = 1}^{n} \log \left(\frac{1}{2}\phi(X_{i}|\theta, \sigma^2 I_{d}) + \frac{1}{2}\phi(X_{i}|-\theta, \sigma^2 I_{d})\right). \label{eq:sample_loglikelihood} 
\end{align}
Here, $\phi(\cdot|\theta, \sigma^2I_{d})$ denotes the density function of multivariate Gaussian distribution with mean $\theta$ and covariance matrix $\sigma^2I_{d}$. 

Similar to GLM, we also consider two regimes of the true parameter: Strong signal-to-noise regime when $\|\theta^{*}\|/\sigma$ is sufficiently large and Low signal-to-noise regime when $\|\theta^{*}\|/ \sigma$ is sufficiently small. To analyze the behaviors of sample NormGD iterates, we define the population version of the maximum likelihood estimation~\eqref{eq:sample_loglikelihood} as follows:
\begin{align}
    \min_{\theta \in \mathbb{R}^{d}} \bar{\mathcal{L}}(\theta) : = - \mathbb{E}\left[\log \left(\frac{1}{2} \phi(X|\theta, \sigma^2 I_{d}) + \frac{1}{2} \phi(X|-\theta, \sigma^2 I_{d})\right)\right]. \label{eq:population_loglikelihood} 
\end{align}
Here, the outer expectation is taken with respect to $X \sim \frac{1}{2} \mathcal{N}(-\theta^{*}, \sigma^2 I_{d}) + \frac{1}{2} \mathcal{N}(\theta^{*}, \sigma^2I_{d})$. We can check that $\bar{\mathcal{L}}$ is also homogeneous in all directions. The strong signal-to-noise regime corresponds to the setting when $\bar{\mathcal{L}}$ is homogeneous with all fast directions while the low signal-to-noise regime is associated with the setting when $\bar{\mathcal{L}}$ is homogeneous with all slow directions.

\vspace{0.5 em}
\noindent
\textbf{Strong signal-to-noise regime:} For the strong signal-to-noise regime, we assume that $\|\theta^{*}\| \geq C \sigma$ for some universal constant $C$. Since the function $\bar{\mathcal{L}}$ is locally strongly convex and smooth as long as $\theta \in \mathbb{B}(\theta^{*}, \frac{\|\theta^{*}\|}{4}) $ (see Corollary 1 in~\cite{Siva_2017}), the Assumption~\ref{assump:homogeneous_property_fast} under the homogeneous setting with all fast directions is satisfied. Furthermore, as long as we choose the radius $r \leq \frac{\|\theta^{*}\|}{4}$ and the sample size $n \geq C_{1} d \log(1/ \delta)$ for some universal constant $C_{1}$, with probability at least $1 - \delta$ there exist universal constants $C_{2}$ and $C_{3}$ such that
\begin{align}
    \sup_{\theta \in \mathbb{B}(\theta^{*}, r)} \| \nabla \bar{\mathcal{L}}_{n}(\theta) - \nabla \bar{\mathcal{L}}(\theta) \| \leq C_{2} \sqrt{\frac{d \log(1/ \delta)}{n}}, \nonumber \\
    \sup_{\theta \in \mathbb{B}(\theta^{*}, r)} \| \nabla^2 \bar{\mathcal{L}}_{n}(\theta) - \nabla^2 \bar{\mathcal{L}}(\theta) \|_{\text{op}}  \leq C_{3} \sqrt{\frac{d \log(1/ \delta)}{n}}.
    \label{eq:concentration_Hessian_mixture_model}
\end{align}
The proof of claims~\eqref{eq:concentration_Hessian_mixture_model} is in Appendix~\ref{sec:proof:uniform_concentration_mixture_model_homogeneous}. In light of Theorem~\ref{theorem:statistical_rate_homogeneous_settings_fast} in Appendix~\ref{sec:homogenenous_fast} for homogeneous settings with all fast directions, the NormGD iterates converge to the final statistical radius $(d/ n)^{1/2}$ after $\log(n)$ iterations (see Corollary~\ref{corollary:mixture_model} for a formal statement of this result). 

\begin{figure*}[t!]
  \centering \includegraphics[width=0.32\textwidth]{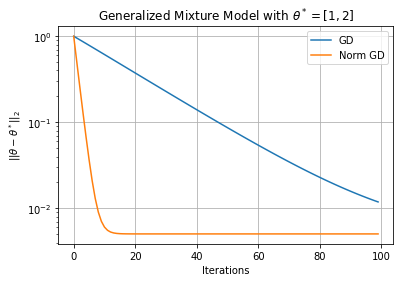}
  \includegraphics[width=0.32\textwidth]{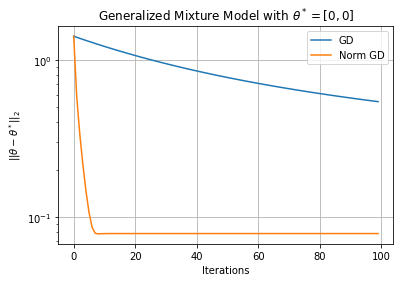}
  \includegraphics[width=0.32\textwidth]{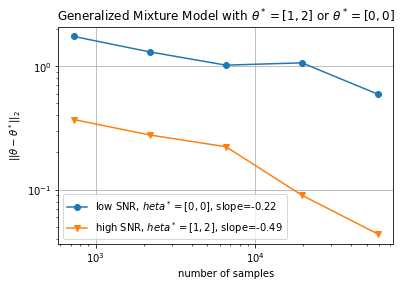}
  \caption{\textit{Verification simulation for the Gaussian Mixture Model (GMM) example}. \textbf{Left}: Both GD and Norm GD converges linearly in the strong signal-to-noise setting; \textbf{Middle}: only Norm GD converges linearly in the low signal-to-noise setting while GD converges sub-linearly; \textbf{Right}: the log-log plot of sample size versus statistical error shows that the statistical error scales with $n^{-0.5}$ in the strong signal-to-noise setting and $n^{-0.25}$ in the low signal-to-noise setting, which coincides with our theory. The slope is computed as the linear regression coefficient of the log sample size versus the log statistical error.}
  \label{fig:GMM-exp}
\end{figure*}

\vspace{0.5 em}
\noindent
\textbf{Low signal-to-noise regime:} Now we move to the low signal-to-noise regime, which refers to the setting when $\|\theta^{*}\|/ \sigma$ is sufficiently small. For the simplicity of computation we specifically assume that $\theta^{*} = 0$. Under this setting, the true model becomes a single Gaussian distribution with mean 0 and covariance matrix $\sigma^2 I_{d}$ while the fitted model~\eqref{eq:overspecify_mixture} has two components with similar weights and symmetric means. This setting is widely referred to as over-specified mixture model, namely, we fit the true mixture model with more components than needed, in statistics and machine learning~\citep{Chen1992, Rousseau-2011}. It is important in practice as the true number of components is rarely known and to avoid underfitting the true model, we tend to use a fitted model with more components than the true number of components. 

In Appendix~\ref{sec:proof:low_signal_mixture_model_homogeneous}, we prove that the population loss function $\bar{\mathcal{L}}$ is homogeneous with all slow directions and satisfy the following properties:
\begin{align}
    \lambda_{\max}(\nabla^2 \bar{\mathcal{L}}(\theta)) & \leq c_{1} \|\theta - \theta^{*}\|^{2}, \label{eq:low_signal_generalized_linear_maximum_eigenvalue} \\
    \lambda_{\min}(\nabla^2 \bar{\mathcal{L}}(\theta)) & \geq c_{2} \|\theta - \theta^{*}\|^{2}, \label{eq:low_signal_generalized_linear_minimum_eigenvalue}
\end{align}
for all $\theta \in \mathbb{B}(\theta^{*}, \frac{\sigma}{2})$ where $c_{1}$ and $c_{2}$ are some universal constants. Therefore, the population loss function $\bar{\mathcal{L}}$ satisfies Assumption~\ref{assump:homogeneous_property} with $\alpha = 2$. 

For the stability of second-order information, we prove in Appendix~\ref{sec:proof:uniform_concentration_mixture_model_homogeneous} that there exist universal constants $C_{1}$ and $C_{2}$ such that for any $r > 0$, with probability $1 - \delta$ as long as $n \geq C_{1} d \log(1/ \delta)$ we obtain
\begin{align}
    \sup_{\theta \in \mathbb{B}(\theta^{*}, r)} \| \nabla^2 \bar{\mathcal{L}}_{n}(\theta) - \nabla^2 \bar{\mathcal{L}}(\theta) \|  \leq C_{2} \sqrt{\frac{d \log(1/ \delta)}{n}}. \label{eq:concentration_Hessian_mixture_model_low_signal}
\end{align}
The uniform concentration bound~\eqref{eq:concentration_Hessian_mixture_model_low_signal} shows that for the low signal-to-noise regime of two-component location Gaussian mixtures, the stability of second-order information in Assumption~\ref{assump:concentration_bound} is satisfied with $\gamma = 0$. Moreover, from Lemma 1 in \citep{Raaz_Ho_Koulik_2018_second} we know $\nabla \bar{\mathcal{L}}_n(\theta^*) = 0$. Combining the results from the homogeneous behaviors of population loss function in equations~\eqref{eq:low_signal_generalized_linear_maximum_eigenvalue}-\eqref{eq:low_signal_generalized_linear_minimum_eigenvalue} and the uniform concentration bound in equation~\eqref{eq:concentration_Hessian_mixture_model_low_signal} to the result of Theorem~\ref{theorem:statistical_rate_homogeneous_settings}, we obtain that the NormGD updates reach the final statistical radius $(d/n)^{1/4}$ after $\log(n)$ number of iterations. 

Now, we would like to formally state the statistical behaviors of the NormGD iterates for both the strong signal-to-noise and low signal-to-noise regimes.
\begin{corollary}
\label{corollary:mixture_model} 
Given the symmetric two-component mixture model~\eqref{eq:overspecify_mixture}, we can find positive universal constants $c, \bar{c}_{1}, \bar{c}_{2}, c_{1}', c_{2}'$ such that with probability at least $1 - \delta$, when $n \geq c d \log(1/ \delta)$ the sequence of NormGD iterates $\{\theta_{n}^{t}\}_{t \geq 0}$ satisfies the following bounds:

(i) When $\|\theta^{*}\| \geq C$ for some sufficiently large constant $C$ and the initialization $\theta_{n}^{0} \in \mathbb{B}(\theta^{*}, \frac{\|\theta\|^{*}}{4})$, we obtain that
\begin{align*}
         \| \theta_{n}^{t} - \theta^{*} \| \leq \bar{c}_{1} \sqrt{\frac{d \log(1/\delta)}{n}}, \quad \quad \text{as long as} \ t \geq \bar{c}_{2} \log \parenth{\frac{n}{d \log(1/\delta)}},
\end{align*}

(ii) Under the setting $\theta^{*} = 0$ and the initialization $\theta_{n}^{0} \in \mathbb{B}(\theta^{*}, \frac{\sigma}{2})$,  we have
\begin{align*}
        \min_{1 \leq k \leq t} \| \theta_{n}^{k} - \theta^{*} \| \leq c_{1}' \parenth{\frac{d \log(1/\delta)}{n}}^{1/4}, \quad \quad \text{for} \ t \geq c_{2}' \log \parenth{\frac{n}{d \log(1/\delta)}}.
\end{align*}
\end{corollary}
We have the following comments with the results of Corollary~\ref{corollary:mixture_model}. In the strong signal-to-noise case, the NormGD algorithm and the fixed step size GD algorithm, which is also the EM algorithm for the symmetric two-component mixture, reach the final statistical radius $(d/n)^{1/2}$ around the true parameter $\theta^{*}$ after $\log(n)$ number of iterations. For the low signal-to-noise regime, the NormGD iterates reach the final statistical radius $(d/n)^{1/4}$ after a logarithmic number of iterations in terms of $n$ while the EM iterates reach that radius after $\sqrt{n}$ number of iterations~\cite{Raaz_Ho_Koulik_2020}. It demonstrates that for fixed dimension $d$ the total computational complexity for the NormGD is at the order of $\mathcal{O}(n)$, which is much cheaper than that of the EM algorithm, which is at the order of $\mathcal{O}(n^{3/2})$. 

\vspace{0.5 em}
\noindent
\textbf{Experiments:}
To verify our theory, we performed simulation on Gaussian Mixture Model (GMM), and the results are shown in Figure~\ref{fig:GMM-exp}. We set $d=2$. For the low signal-to-noise setting, we set $\theta^*$ to be $[0,0]$, and for strong signal-to-noise setting, we set $\theta^*$ to be $[1,2]$. For the left and the middle plots in Figure~\ref{fig:GLM-exp}, the sample size is set to be $10000$.  As in the left plot, when in the strong signal-to-noise setting, both the fixed step size Gradient Descent method (referred to as GD, and is essentially EM algorithm as described above) and our proposed Normalized Gradient Descent method (referred to as NormGD) converges linearly. However, once we shift from the strong signal-to-noise setting to the low signal-to-noise setting, only Norm GD converges linearly, while GD converges only sub-linearly, as shown in the middle plot. To further verify our corollaries, especially how the statistical error scales with $n$, we plot the statistical error versus sample size in the right plot. The experiments were repeated for $10$ times and the average of the statistical error is shown. The slope is computed as the linear regression coefficient of the log sample size versus the log statistical error. As in this log-log plot, in the strong signal-to-noise setting, the statistical error roughly scales with $n^{-0.5}$, while in the low signal-to-noise setting, the statistical error roughly scales with $n^{-0.25}$. This coincides with our theory as in Corollary~\ref{corollary:mixture_model}.

\section{Conclusion}
\label{sec:discussion}
In this paper, we show that by utilizing second-order information in the design of optimization algorithms, we are able to improve the computational complexity of these algorithms for solving parameter estimation in statistical models. In particular, we study the statistical and computational complexities of the NormGD algorithm, a variant of gradient descent algorithm whose step size is scaled by the maximum eigenvalue of the Hessian matrix of the loss function. We show that when the population loss function is homogeneous, the NormGD algorithm only needs a logarithmic number of iterations to reach the final statistical radius around the true parameter. In terms of iteration complexity and total computational complexity, it is cheaper than fixed step size GD algorithm, which requires a polynomial number of iterations to reach the similar statistical radius under the singular statistical model settings.
\section{Acknowledgements}
\label{sec:acknowledge}
This work was partially supported by the NSF IFML 2019844 award and research gifts by UT Austin ML grant to NH, and by NSF awards 1564000 and 1934932 to SS.
\appendix


\begin{center}
{\bf \Large Supplement to ``Improving Computational Complexity in Statistical Models
    with Second-Order Information''}
\end{center}

In the supplementary material, we collect proofs and results deferred from the main text. In Appendix~\ref{sec:homogenenous_fast}, we provide general theory for the statistical guarantee of NormGD for the homogeneous settings with all fast directions of the population loss function. In Appendix~\ref{sec:proof_main}, we provide proofs for the main results in the main text. We then provide proofs for the statistical and computational complexities of NormGD under generalized linear models and mixture models respectively in Appendices~\ref{sec:proof:generalized_linear_models} and~\ref{sec:proof:mixture_model}. 
\section{Homogeneous Settings with All Fast Directions}
\label{sec:homogenenous_fast}
In this Appendix, we provide statistical guarantee for the NormGD iterates when the population loss function is homogeneous with all fast directions. Following the population to sample analysis in equation~\eqref{eq:population_to_sample}, we first consider the strong convexity and Lipschitz smoothness assumptions that characterize all fast directions.
\begin{enumerate}[label=(S.1)] 
\item \label{assump:homogeneous_property_fast} 
(Strongly convexity and Lipschitz smoothness) For some radius $r > 0$, for all $\theta\in \mathbb{B}(\theta^*, r)$ we have
\begin{align*}
    \bar{c}_{1} \leq \lambda_{\min}(\nabla^2 f(\theta)) \leq \lambda_{\max}(\nabla^2 f(\theta)) \leq \bar{c}_{2},
\end{align*}
where $\bar{c}_{1} > 0$ and $\bar{c}_2>0$ are some universal constants depending on $r$.
\end{enumerate}
The Assumption~\ref{assump:homogeneous_property_fast} is a special case of Assumption~\ref{assump:homogeneous_property} when $\alpha = 0$. A simple example for the function $f$ that satisfies Assumption~\ref{assump:homogeneous_property_fast} is $f(\theta) = \|\theta\|^2$.

Given the Assumption~\ref{assump:homogeneous_property_fast}, we obtain the following result for the contraction of the population NormGD operator $F$ around the true parameter $\theta^{*}$.
\begin{lemma} \label{lemma:optimization_homogeneous_fast}
Assume Assumption~\ref{assump:homogeneous_property_fast} holds for some universal constants $\bar{c}_1, \bar{c}_2$. Then, if the step-size $\eta \leq \frac{\bar{c}_2^2}{2\bar{c}_1^2}$, then we have that
\begin{align*}
    \|F(\theta)) - \theta^*\| \leq \bar{\kappa} \| \theta - \theta^*\|,
\end{align*}
where $\bar{\kappa} < 1$ is a universal constant that only depends on $\eta, \bar{c}_1, \bar{c}_2$.
\end{lemma}
The proof of Lemma~\ref{lemma:optimization_homogeneous_fast} is a direct from the proof of Lemma~\ref{lemma:optimization_homogeneous} with $\alpha = 0$; therefore, its proof is omitted.
\begin{enumerate}[label=(S.2)] 
\item\label{assump:concentration_bound_fast} (Stability of first and second-order information) For some fixed positive parameter $r > 0$, there exist a noise function $\varepsilon: \mathbb{N} \times (0,1] \to \mathbb{R}^{+}$, and universal constants $\bar{c}_3, \bar{c}_4 > 0$ depends on $r$, such that 
\begin{align*}
    \sup_{\theta\in \mathbb{B}(\theta^*, r)} \|\nabla f_{n}(\theta) - \nabla f(\theta)\| & \leq \bar{c}_3 \cdot \varepsilon(n, \delta), \\
    \sup_{\theta\in \mathbb{B}(\theta^*, r)} \|\nabla^{2} f_{n}(\theta) - \nabla^2 f(\theta)\|_{\text{op}} & \leq \bar{c}_4 \cdot \varepsilon(n, \delta).
\end{align*}
for all $r \in (0, r)$ with probability $1 - \delta$.
\end{enumerate}
We would like to remark that the assumption in the uniform concentration of $\nabla f_{n}(\theta)$ around $\nabla f(\theta)$ is standard for analyzing optimization algorithms for solving parameter estimation under locally strongly convex and smooth population loss function~\cite{Siva_2017, Ho_Instability}. The extra assumption on the uniform concentration of the empirical Hessian matrix $\nabla^2 f_{n}(\theta)$ around the population Hessian matrix $\nabla^2 f(\theta)$ is to ensure that $\lambda_{\max}(\nabla^2 f_{n}(\theta)$ in NormGD algorithm will stay close to $\lambda_{\max}(\nabla^2 f(\theta))$. These two conditions are sufficient to guarantee the stability of the sample NormGD operator $F_{n}$ around the population NormGD operator in the following lemma.
\begin{lemma} \label{lemma:concentration_bound_population_sample_fast}
Assume that Assumption~\ref{assump:concentration_bound_fast} holds, and $n$ is sufficiently large such that $\bar{c}_1 > 2\bar{c_3} \varepsilon(n, \delta)$. Then, we obtain that
\begin{align*}
    \sup_{\theta \in \mathbb{B}(\theta^{*}, r)} \|F_n(\theta)-F(\theta)\| \leq \bar{c}_5 \varepsilon(n, \delta),
\end{align*}
and $\bar{c}_5$ is a universal constant depends on $\eta, \bar{c}_1, \bar{c}_2, \bar{c}_3, \bar{c}_4$.
\end{lemma}
\begin{proof}
With straightforward calculation, we have that
\begin{align*}
    \|F_n(\theta) - F(\theta)\|
    \leq & \eta \left(\left\|\frac{\nabla f(\theta)(\lambda_{\max}(\nabla^2 f(\theta)) - \lambda_{\max}(\nabla^2 f_n(\theta)))}{\lambda_{\max}(\nabla^2 f_n(\theta)) \lambda_{\max}(\nabla^2 f(\theta))}\right\| + \left\|\frac{\nabla f_n(\theta) - \nabla f(\theta)}{\lambda_{\max}(\nabla^2 f_n(\theta))}\right\| \right)\\
    \leq & \eta \left(\frac{\bar{c}_2\bar{c}_3\varepsilon(n, \delta)}{(\bar{c}_1 - \bar{c}_3 \varepsilon(n, \delta)) \bar{c}_1 } + \frac{\bar{c}_4 \varepsilon(n, \delta)}{\bar{c}_1  - \bar{c}_3 \varepsilon(n, \delta)}\right)\\
    \leq & \eta \left(\frac{2\bar{c}_2 \bar{c}_3 + 2\bar{c}_1\bar{c_4}}{\bar{c_1}^2}\right) \varepsilon(n, \delta).
\end{align*}
Take $\bar{c}_5$ accordingly, we conclude the proof.
\end{proof}
\begin{theorem} \label{theorem:statistical_rate_homogeneous_settings_fast}
Assume Assumptions~\ref{assump:homogeneous_property_fast} and~\ref{assump:concentration_bound_fast} hold, and $n$ is sufficient large such that $\bar{c}_1 > 2\bar{c}_3 \varepsilon(n, \delta)$ and $\bar{c}_5 \varepsilon(n, \delta) \leq (1 - \bar{\kappa}) r$ where $\bar{\kappa}$ is the constant defined in Lemma~\ref{lemma:optimization_homogeneous_fast}.
Then, there exist universal constants $\bar{C}_{1}$, $\bar{C}_{2}$ such that for $t \geq \bar{C}_{1}\log (1/\varepsilon(n, \delta))$, the following holds:
\begin{align*}
    \|\theta_n^t - \theta^*\| \leq \bar{C}_{2} \cdot \varepsilon(n, \delta),
\end{align*}
\end{theorem}
\begin{proof}
With the triangle inequality, we have that
\begin{align*}
    \|\theta_n^{t+1} - \theta^*\| = & \|F_n(\theta_n^t) - \theta^*\|\\
    \leq & \|F_n(\theta_n^t) - F(\theta_n^t)\| + \|F(\theta_n^t) - \theta^*\| \\
    \leq & \sup_{\theta\in \mathbb{B}(\theta^*, r)}\|F_n(\theta) - F(\theta)\| + \bar{\kappa} \|\theta_n^t - \theta^*\| \\
    \leq & \bar{c}_5 \varepsilon(n, \delta) + \bar{\kappa} r \leq r.
\end{align*}
Hence, we know $\|\theta_n^t - \theta^*\| \leq r$ for all $t \in \mathbb{N}$. Furthermore, by repeating the above argument $T$ times, we can obtain that 
\begin{align*}
    \|\theta_n^T - \theta^*\| \leq & \bar{c}_5 \varepsilon(n, \delta) \left(\sum_{t=0}^{T-1}\bar{\kappa}^t\right) + \bar{\kappa}^{T}\|\theta_n^0 - \theta^*\|\\
    \leq & \frac{\bar{c}_5}{1-\bar{\kappa}} \varepsilon(n, \delta) + \bar{\kappa}^T r.
\end{align*}
By choosing $T \leq \frac{\log(r) + \log(1/\varepsilon(n, \delta))}{\log (1/\bar{\kappa})}$, we know $\bar{\kappa}^{T} r \leq \varepsilon(n, \delta)$, hence
\begin{align*}
      \|\theta_n^T - \theta^*\| \leq \left(\frac{\bar{c}_5}{1-\bar{\kappa}} + 1 \right) \varepsilon(n, \delta).
\end{align*}
Take $\bar{C}_1$, $\bar{C}_2$ accordingly, we conclude the proof.
\end{proof}
\section{Proofs of Main Results}
\label{sec:proof_main}
In this Appendix, we provide proofs for the results in the main text.
\subsection{Proof of Lemma~\ref{lemma:optimization_homogeneous}}
\label{sec:proof:theorem:optimization_homogeneous}
We start from the following lemma:
\begin{lemma}
\label{lem:obj_param_conversion}
Assume Assumption~\ref{assump:homogeneous_property} holds, we have that
\begin{align*}
    f(\theta) - f(\theta^*) \geq \frac{c_1\|\theta - \theta^*\|^{\alpha + 2}}{(\alpha + 1)(\alpha + 2)}.
\end{align*}
\end{lemma}
\begin{proof}
Consider $g(\theta) = f(\theta) - \frac{c_1\|\theta - \theta^*\|^{\alpha + 2}}{(\alpha + 1)(\alpha + 2)}$. With Assumption~\ref{assump:homogeneous_property}, we know that 
\begin{align*}
    \nabla^2 g(\theta) = \nabla^2 f(\theta) -  \frac{c_1}{(\alpha + 1)(\alpha + 2)}\left(\alpha(\alpha + 2) \|\theta-\theta^*\|^{\alpha - 2} (\theta-\theta)^* (\theta-\theta^*)^\top + (\alpha + 2)\|\theta - \theta^*\|^{\alpha} I\right) \succeq{0},
\end{align*}
as the operator norm of $\alpha(\alpha + 2) \|\theta - \theta^*\|^{\alpha - 2} (\theta-\theta^*)(\theta-\theta^*)^\top + (\alpha + 2)\|\theta\|^{\alpha} I $ is less than $(\alpha+1)(\alpha + 2) \|\theta - \theta^*\|^{\alpha}$. Meanwhile, we have that
\begin{align*}
    \nabla g(\theta) = \nabla f(\theta) - \frac{c_1 \|\theta-\theta^*\|^{\alpha}}{\alpha + 1} (\theta - \theta^*).
\end{align*}
As $\nabla f(\theta^*) = 0$, we know $\nabla g(\theta^*) = 0$, which means $\theta^*$ is the minimizer of $g$. Hence,
\begin{align*}
    f(\theta^*) = g(\theta^*) \leq g(\theta) = f(\theta) - \frac{c_1\|\theta - \theta^*\|^{\alpha + 2}}{(\alpha + 1)(\alpha + 2)},
\end{align*}
which means
\begin{align*}
    f(\theta) - f(\theta^*) \geq \frac{c_1\|\theta - \theta^*\|^{\alpha + 2}}{(\alpha + 1)(\alpha + 2)}.
\end{align*}
As a consequence, we obtain the conclusion of Lemma~\ref{lem:obj_param_conversion}. 
\end{proof}
\noindent
Now, we prove Lemma~\ref{lemma:optimization_homogeneous}. Notice that
\begin{align*}
    \|F(\theta) - \theta^*\|^2 = & \left\|\theta - \frac{\eta}{\lambda_{\max}(\nabla^2 f(\theta))} \nabla f(\theta) - \theta^*\right\|^2\\
    = & \|\theta - \theta^*\|^2 - \frac{2\eta}{\lambda_{\max}(\nabla^2 f(\theta))} \langle \nabla f(\theta), \theta - \theta^*\rangle + \frac{\eta^2}{\lambda_{\max}^2(\nabla^2(f(\theta)))}\|\nabla f(\theta)\|^2\\
    = & \|\theta - \theta^*\|^2 - \frac{\eta}{\lambda_{\max}(\nabla^2 f(\theta))}\left( 2\langle \nabla f(\theta), \theta - \theta^*\rangle - \frac{\eta}{\lambda_{\max}(\nabla^2 f(\theta))}\|\nabla f(\theta)\|^2\right)\\
    \leq & \|\theta - \theta^*\|^2 - \frac{\eta}{\lambda_{\max}(\nabla^2 f(\theta))}\left( 2(f(\theta) - f(\theta^*)) - \frac{\eta}{\lambda_{\max}(\nabla^2 f(\theta))}\|\nabla f(\theta)\|^2\right),
\end{align*}
where the last inequality is due to the convexity. With Assumption \ref{assump:homogeneous_property}, we have that
\begin{align*}
    \|\nabla f(\theta)\| = & \left\|\int_0^1 \nabla^2 f(\theta^* + t(\theta - \theta^*))(\theta - \theta^*) d t\right\|\\
    \leq & \int_0^1\left\| \nabla^2 f(\theta^* + t(\theta - \theta^*))(\theta - \theta^*) \right\|d t\\
    \leq & \int_0^1 \lambda_{\max}(\nabla^2 f(\theta^* + t(\theta - \theta^*)))\|\theta - \theta^*\|d t\\
    \leq & \int_0^1 c_2 t^{\alpha}\|(\theta - \theta^*)\|^{\alpha}\|\theta - \theta^*\|d t\\
    \leq & \frac{c_2}{\alpha + 1} \|\theta - \theta^*\|^{\alpha + 1}.
\end{align*}
\noindent
As $\eta\leq \frac{c_1^2}{2c_2^2} \leq \frac{c_1^2(\alpha + 1)}{c_2^2(\alpha + 2)}$, we have that
\begin{align*}
    \frac{\eta}{\lambda_{\max} (\nabla^2 f(\theta))}\left(f(\theta) - f(\theta^*) - \frac{\eta}{\lambda_{\max} (\nabla^2 f(\theta))} \|\nabla f(\theta)\|^2\right) & \\
    & \hspace{- 6 em} \geq  \frac{\eta}{c_2}\left(\frac{c_1}{(\alpha + 1)(\alpha + 2)} - \frac{\eta c_2^2}{c_1(\alpha + 1)^2}\right)\|\theta - \theta^*\|^2.
\end{align*}
Hence, we find that
\begin{align*}
    \|F(\theta) - \theta^*\|^2 \leq \left(1 - \frac{\eta}{c_2}\left(\frac{c_1}{(\alpha + 1)(\alpha + 2)} - \frac{\eta c_2^2}{c_1(\alpha + 1)^2}\right)\right)\|\theta - \theta^*\|^2.
\end{align*}
Take $\kappa$ accordingly, we conclude the proof.
\subsection{Proof of Lemma~\ref{lemma:concentration_bound_population_sample}}
\label{sec:proof:lemma:concentration_bound_population_sample}
Notice that
\begin{align*}
    \|F_n(\theta) - F(\theta)\| = & \left\|\frac{\eta}{\lambda_{\max}(\nabla^2 f_n(\theta))}\nabla f_n(\theta) - \frac{\eta}{\lambda_{\max}(\nabla^2 f(\theta))} \nabla f(\theta)\right\|\\
    = & \eta \left\|\frac{\nabla f_n(\theta) \lambda_{\max}(\nabla^2 f(\theta)) - \nabla f(\theta) \lambda_{\max}(\nabla^2 f_n(\theta))}{\lambda_{\max}(\nabla^2 f_n(\theta)) \lambda_{\max}(\nabla^2 f(\theta))}\right\|\\
    \leq & \eta \left(\left\|\frac{\nabla f(\theta)(\lambda_{\max}(\nabla^2 f(\theta)) - \lambda_{\max}(\nabla^2 f_n(\theta)))}{\lambda_{\max}(\nabla^2 f_n(\theta)) \lambda_{\max}(\nabla^2 f(\theta))}\right\| + \left\|\frac{\nabla f_n(\theta) - \nabla f(\theta)}{\lambda_{\max}(\nabla^2 f_n(\theta))}\right\| \right).
\end{align*}
For the term $\|\nabla f_n(\theta) - \nabla f(\theta)\|$, we have that
\begin{align*}
    \|\nabla f_n(\theta) - \nabla f(\theta)\| \leq & \|\nabla f_n(\theta^*) - \nabla f(\theta^*)\| \\
    + & \left\|\int_{0}^{1} (\nabla^2 f_n(\theta^* + t(\theta - \theta^*)) - \nabla^2 f(\theta^* + t(\theta - \theta^*)) )(\theta - \theta^*) dt\right\|\\
    \leq & \int_{0}^1\|(\nabla^2 f_n(\theta^* + t(\theta - \theta^*)) - \nabla^2 f(\theta^* + t(\theta - \theta^*)) )(\theta - \theta^*)\| dt\\
    \leq & \int_{0}^1\|\nabla^2 f_n(\theta^* + t(\theta - \theta^*)) - \nabla^2 f(\theta^* + t(\theta - \theta^*))\|_{\mathrm{op}}\|\theta - \theta^*\| dt \\
    \leq & \int_{0}^1 c_3 t^{\gamma} \varepsilon(n, \delta)\|\theta - \theta^*\|^{\gamma+1} dt\\
    = & \frac{c_3\|\theta - \theta^*\|^{\gamma + 1}\varepsilon(n, \delta)}{\gamma + 1}.
\end{align*}
Meanwhile, it's straightforward to show that
\begin{align*}
    |\lambda_{\max}(\nabla^2 f_n(\theta)) - \lambda_{\max}(\nabla f(\theta))|\leq 3c_3 r^{\gamma} \varepsilon(n, \delta).
\end{align*}
Hence, we have that
\begin{align*}
    \|F_n(\theta) - F(\theta)\|
    \leq & \eta \left(\left\|\frac{\nabla f(\theta)(\lambda_{\max}(\nabla^2 f(\theta)) - \lambda_{\max}(\nabla^2 f_n(\theta)))}{\lambda_{\max}(\nabla^2 f_n(\theta)) \lambda_{\max}(\nabla^2 f(\theta))}\right\| + \left\|\frac{\nabla f_n(\theta) - \nabla f(\theta)}{\lambda_{\max}(\nabla^2 f_n(\theta))}\right\| \right)\\
    \leq & \eta \left(\frac{3c_2 c_3  r^{\gamma + 1 - \alpha} \varepsilon(n, \delta)}{(\alpha + 1)(c_1 r^\alpha - 3c_3 r^\gamma \varepsilon(n, \delta)) c_1 r^{\alpha}} + \frac{c_3 r^\gamma \varepsilon(n, \delta)}{(\gamma + 1)(c_1 r^\alpha - 3c_3 r^{\gamma+1} \varepsilon(n, \delta))}\right).
\end{align*}
As $r\geq \left(\frac{6c_3 \varepsilon(n, \delta)}{c_1}\right)^{1/(\alpha - \gamma)}$, we can further have
\begin{align*}
    \|F_n(\theta) - F(\theta)\|
    \leq & \eta \left(\frac{6c_2 c_3}{(\alpha + 1) c_1^2}+\frac{2c_3}{(\gamma + 1)c_1}\right) r^{\gamma + 1 - \alpha} \varepsilon(n, \delta).
\end{align*}
Taking $c_4$ accordingly, we conclude the proof.
\section{Proof of Generalized Linear Models}
\label{sec:proof:generalized_linear_models}
In this appendix, we provide the proof for the NormGD in generalized linear models.
\subsection{Homogeneous assumptions}
\label{sec:proof:low_signal_generalized_linear_homogeneous}
Based on the formulation of the population loss function $\mathcal{L}$ in equation~\eqref{eq:population_no_signal}, we have
\begin{align*}
    &\nabla \mathcal{L}(\theta) = 2p(2p-1)!!(\theta-\theta^*) \|\theta - \theta^{*}\|^{2p - 2}, \\
    & \nabla^2 \mathcal{L}(\theta) = (2p(2p-1)!!)\|\theta - \theta^*\|^{2p-4} \left(\|\theta - \theta^*\|^2 I_{d} + (2p-4)(\theta - \theta^*)(\theta - \theta^*)^\top\right).
\end{align*}
Notice that, $\theta - \theta^*$ is an eigenvector of $\|\theta - \theta^*\|^2 I_{d} + (2p-4)(\theta - \theta^*)(\theta - \theta^*)^\top $ with eigenvalue $(2p-3)\|\theta - \theta^*\|^2$, and any vector that is orthogonal to $\theta - \theta^*$ (which forms a $d-1$ dimensional subspace) is an eigenvector of $\|\theta - \theta^*\|^2 I_{d} + (2p-4)(\theta - \theta^*)(\theta - \theta^*)^\top$ with eigenvalue $\|\theta - \theta^*\|^2$. Hence, we have that
\begin{align*}
    \lambda_{\max}(\|\theta - \theta^*\|^2 I_{d} + (2p-4)(\theta - \theta^*)(\theta - \theta^*)^\top) = & (2p-3)\|\theta - \theta^*\|^2,\\
    \lambda_{\min}(\|\theta - \theta^*\|^2 I_{d} + (2p-4)(\theta - \theta^*)(\theta - \theta^*)^\top) = & \|\theta - \theta^*\|^2,
\end{align*}
which shows that $\mathcal{L}(\theta)$ satisfies the homogeneous assumption.

\subsection{Uniform concentration bound}
\label{sec:proof:low_signal_generalized_linear_uniform_concentration}
The proof for the concentration bound~\eqref{eq:generalized_linear_concentration_first_order} is in Appendix D.1 of~\citep{Tongzheng_2022}; therefore, it is omitted. We focus on proving the uniform concentration bounds~\eqref{eq:generalized_linear_concentration_second_order} and~\eqref{eq:concentration_Hessian_generalized_model} for the Hessian matrix $\nabla^2 \mathcal{L}_{n}(\theta)$ around the Hessian matrix $\nabla^2 \mathcal{L}(\theta)$ under both the strong and low signal-to-noise regimes. Indeed, we would like to show the following uniform concentration bound that captures both the bounds~\eqref{eq:generalized_linear_concentration_second_order} and~\eqref{eq:concentration_Hessian_generalized_model}.
\begin{lemma}
\label{lemma:unified_concentration_bound_generalized_linear}
There exist universal constants $C_{1}$ and $C_{2}$ such that as long as $n \geq C_{1} (d \log(d/ \delta))^{2p}$ we obtain that
\begin{align}
    \sup_{\theta \in \mathbb{B}(\theta^{*},r)} \|\nabla^2 \mathcal{L}_n(\theta) - \nabla^2 \mathcal{L}(\theta)\|_{\text{op}} \leq C_{2} \parenth{(r + \|\theta^{*}\|)^{p - 2} + (r + \|\theta^{*}\|)^{2p - 2}} \sqrt{\frac{d + \log(1/\delta)}{n}}. \label{eq:unified_concentration_bound_generalized_linear}
\end{align}
\end{lemma}
\begin{proof}[Proof of Lemma~\ref{lemma:unified_concentration_bound_generalized_linear}]
Direct calculation shows that
\begin{align*}
    \nabla^2 \mathcal{L}_n(\theta) = & \frac{1}{n}\sum_{i=1}^n\left(p(X_i^\top \theta)^{2p-2} - p(p-1) Y_i (X_i^\top \theta)^{p-2} \right) X_i X_i^\top,\\
    \nabla^2 \mathcal{L}(\theta) = & \mathbb{E}\left[ p(X^\top \theta)^{2p-2}- p(p-1)(X^\top \theta^*)^p (X^\top \theta)^{p-2} XX^\top\right].
\end{align*}
Therefore, we obtain
\begin{align*}
     \nabla^2 \mathcal{L}_n(\theta) - \nabla^2 \mathcal{L}(\theta) = & p\left(\frac{1}{n} \sum_{i=1}^n (X_i^\top \theta)^{2p-2}X_iX_i^\top - \mathbb{E}\left[(X^\top \theta)^{2p-2} XX^\top\right]\right) \\
     & - p(p-1) \left(\frac{1}{n}\sum_{i=1}^n Y_i (X_i^\top \theta)^{p-2}  X_i X_i^\top - \mathbb{E}\left[(X^\top \theta^*)^p(X^\top \theta)^{p-2} XX^\top\right]\right).
\end{align*}
Using the triangle inequality with the operator norm, the above equation leads to
\begin{align}
    \sup_{\theta \in \mathbb{B}(\theta^{*}, r)} \|\nabla^2 \mathcal{L}_n(\theta) - \nabla^2 \mathcal{L}(\theta)\|_{\text{op}} \leq C (A_{1} + A_{2} + A_{3}), \label{eq:key_inequality_generalized_linear_first}
\end{align}
where $C$ is some universal constant and $A_{1}, A_{2}, A_{3}$ are defined as follows:
\begin{align}
    A_{1} & = \sup_{\theta \in \mathbb{B}(\theta^{*}, r)} \biggr \|\frac{1}{n} \sum_{i=1}^n (X_i^\top \theta)^{2p-2}X_iX_i^\top - \mathbb{E}\left[(X^\top \theta)^{2p-2} XX^\top\right] \biggr \|_{\text{op}}, \nonumber \\
    A_{2} & = \sup_{\theta \in \mathbb{B}(\theta^{*}, r)} \biggr \|\frac{1}{n}\sum_{i=1}^n (Y_i - (X_{i}^{\top} \theta^{*})^{p}) (X_i^\top \theta)^{p-2}  X_i X_i^\top \biggr \|_{\text{op}}, \nonumber \\
    A_{3} & =  \sup_{\theta \in \mathbb{B}(\theta^{*}, r)} \biggr \| \frac{1}{n}\sum_{i=1}^n (X_{i}^{\top} \theta^{*})^{p} (X_i^\top \theta)^{p-2}  X_i X_i^\top - \mathbb{E}\left[(X^\top \theta^*)^p(X^\top \theta)^{p-2} XX^\top\right] \biggr \|_{\text{op}}. \label{eq:key_inequality_generalized_linear_second}
\end{align}
With variational characterization of the operator norm and upper bound the norm of any $\theta \in \mathbb{B}(\theta^*, r)$ with $r + \|\theta^{*}\|$, we have 
\begin{align*}
    A_{1} & \leq (r +\|\theta^{*}\|)^{2p - 2} T_{1}, \\
    A_{2} & \leq (r +\|\theta^{*}\|)^{p - 2} T_{2}, \\
    A_{3} & \leq (r +\|\theta^{*}\|)^{p - 2} T_{3},
\end{align*}
where the terms $T_{1}, T_{2}, T_{3}$ are defined as follows:
\begin{align*}
    & T_1:= \sup_{u\in\mathbb{S}^{d-1}, \theta\in\mathbb{S}^{d-1}}\left|\frac{1}{n} \sum_{i=1}^n (X_i^\top \theta)^{2p-2}(X_i^\top u)^2 - \mathbb{E}\left[(X^\top \theta)^{2p-2} (X^\top u)^2\right]\right|\\
    & T_2:= \sup_{u\in\mathbb{S}^{d-1}, \theta\in\mathbb{S}^{d-1}}\left|\frac{1}{n} \sum_{i=1}^n (Y_i - (X_i^\top \theta^*)^p) (X_i^\top \theta)^{p-2}(X_i^\top u)^2\right|\\
    & T_3:=\sup_{u\in\mathbb{S}^{d-1}, \theta\in\mathbb{S}^{d-1}}\left|\frac{1}{n} \sum_{i=1}^n  (X_i^\top \theta^*)^p (X_i^\top \theta)^{p-2}(X_i^\top u)^2 - \mathbb{E}\left[(X^\top \theta^*)^p(X^\top \theta)^{p-2} (X^\top u)^2\right]\right|,
\end{align*}
where $\mathbb{S}^{d-1}$ is the unit sphere in $\mathbb{R}^d$. We know consider the high probability bound of each individual terms following the proof strategy in \citep{Tongzheng_2022}.
\paragraph{Bound for $T_2$:} Assume $U$ is a $1/8$-cover of $\mathbb{S}^{d-1}$ under $\|\cdot\|_2$ with at most $17^d$ elements, the standard discretization arguments (e.g. Chapter 6 in \citep{Wainwright_nonasymptotic}) show
\begin{align*}
    \sup_{u\in\mathbb{S}^{d-1}}\left|\frac{1}{n} \sum_{i=1}^n (Y_i - (X_i^\top \theta^*)^p) (X_i^\top \theta)^{p-2}(X_i^\top u)^2\right| \leq & 2\sup_{u\in U}\left|\frac{1}{n} \sum_{i=1}^n (Y_i - (X_i^\top \theta^*)^p) (X_i^\top \theta)^{p-2}(X_i^\top u)^2\right|.
\end{align*}
With a symmetrization argument, we know for any even integer $q \geq 2$,
\begin{align*}
    & \mathbb{E}\left(\sup_{u\in \mathbb{S}^{d-1}}\left|\frac{1}{n} \sum_{i=1}^n (Y_i - (X_i^\top \theta^*)^p) (X_i^\top \theta)^{p-2}(X_i^\top u)^2\right|\right)^q \\
    \leq & \mathbb{E}\left(\sup_{u\in\mathbb{S}^{d-1}}\left|\frac{2}{n} \sum_{i=1}^n \varepsilon_i(Y_i - (X_i^\top \theta^*)^p)(X_i^\top \theta)^{p-2} (X_i^\top u)^2\right|\right)^q,
\end{align*}
where $\{\varepsilon_i\}_{i\in [n]}$ is a set of i.i.d. Rademacher random variables. Furthermore, for a compact set $\Omega$, we define
\begin{align*}
    \mathcal{R}(\Omega) := \sup_{\theta\in\Omega, p^\prime\in[1, p]} \left|\frac{2}{n}\sum_{i=1}^n \varepsilon_i(Y_i - (X_i^\top \theta^*)^p)(X_i^\top \theta)^{p^\prime-2} (X_i^\top u)^2\right|,
\end{align*}
where $\mathcal{N}(t)$ is a $t$-cover of $\mathbb{S}^{d-1}$ under $\|\cdot\|_2$. Then we have
\begin{align*}
    \mathcal{R}(\mathbb{S}^{d-1}) \leq & 2\mathcal{R}(\mathcal{N}(t)) + 3^{p - 2} t \mathcal{R}(\mathbb{S}^{d-1}).
\end{align*}
By taking $t=3^{-p + 1}$, we obtain that $\mathcal{R}(\mathbb{S}^{d-1}) \leq 3 \mathcal{R}(\mathcal{N}(3^{-p + 1}))$. Furthermore, with the union bound, for any $q\geq 1$ we have that
\begin{align*}
    \sup_{\theta\in\mathbb{S}^{d-1}, p^\prime\in [2, p]} \mathbb{E}\left[\left|\frac{2}{n}\sum_{i=1}^n \varepsilon_i (Y_i -(X_i^\top \theta^*)^p((X_i^\top \theta)^{p^\prime-2} (X_i^\top u)^2\right|^q\right] 
    \leq \frac{\mathbb{E}[\mathcal{R}^q(\mathcal{N}(3^{-p + 1}))]}{p|\mathcal{N}(3^{-p + 1})|}.
\end{align*}
Therefore, we only need to consider $\mathbb{E}\left[\left|\frac{2}{n}\sum_{i=1}^n \varepsilon_i (Y_i - (X_i^\top \theta^*)^p)(X_i^\top \theta)^{p^\prime-2} (X_i^\top u)^2\right|^q\right] $. An application of Khintchine’s inequality \citep{boucheron2013concentration} demonstrates that we can find a universal constant $C$ such that for all $p^\prime\in[2, p]$, we have
\begin{align*}
    \mathbb{E}\left[\left|\frac{2}{n}\sum_{i=1}^n \varepsilon_i (Y_i -(X_i^\top \theta^*)^p)(X_i^\top \theta)^{p^\prime-2} (X_i^\top u)^2\right|^q\right] & \\
    &  \hspace{- 5 em} \leq \mathbb{E}\left[\left(\frac{Cq}{n^2}\sum_{i=1}^n (Y_i - (X_i^\top \theta^*)^{p})^2(X_i^\top \theta)^{2(p^\prime-2)} (X_i^\top u)^4\right)^{q/2}\right].
\end{align*}
From the assumptions on $Y_{i}$ and $X_{i}$ for all $i \in [n]$, we have
\begin{align*}
    \mathbb{E}\left[(Y_i - (X_i^\top \theta^*)^{p})^2(X_i^\top \theta)^{2(p^\prime-2)}(X_i^\top u)^4\right] \leq & (2p^\prime)^{ p^\prime},\\
    \mathbb{E}\left[\left((Y - (X_i^\top \theta^*)^{p})^2(X_i^\top \theta)^{2(p^\prime-4)}(X_i^\top u)^4\right)^{q/2}\right] \leq & (2p^\prime q)^{ p^\prime q}.
\end{align*}
From Lemma 2 in \citep{mou2019diffusion}, we have
\begin{align*}
    & \hspace{- 2 em} \left|\frac{1}{n}\sum_{i=1}^n \left((Y - (X_i^\top \theta^*)^{p})^2(X_i^\top \theta)^{2(p^\prime-2)} (X_i^\top u)^4\right)^{q/2} - \mathbb{E}\left[(Y_i - (X_i^\top \theta^*)^{p})^2 (X_i^\top \theta)^{2(p^\prime-2)}(X_i^\top u)^4\right]\right|\\
    & \leq (8p^\prime)^{p^\prime} \sqrt{\frac{\log 4/\delta}{n}} + (2p^\prime \log (n/\delta))^{p^\prime} \frac{\log 4/\delta}{n}
\end{align*}
with probability at least $1-\delta$. Hence, we find that
\begin{align*}
    & \hspace{- 5 em} \mathbb{E}\left[\left(\frac{1}{n}\sum_{i=1}^n(Y_i - (X_i^\top \theta^*)^{p})^2(X_i^\top \theta)^{2(p^\prime-1)} (X_i^\top u)^4\right)^{q/2}\right]\\
    \leq & (4p^\prime)^{p^\prime q} + 2^{q/2}\int_{0}^{\infty} \mathbb{P}\biggr[\biggr|\sum_{i=1}^n \left((Y_i - (X_i^\top \theta^*)^{p})^2 (X_i^\top \theta)^{2(p^\prime-2)} (X_i^\top u)^4\right)^{q/2} \\
    & - \mathbb{E}\left[(Y_i - (X_i^\top \theta^*)^{p})^2(X_i^\top \theta)^{2(p^\prime-2)}(X_i^\top u)^4\right]\biggr|\geq \lambda\biggr] d\lambda^{q/2} \\
    \leq & (4p^\prime )^{p^\prime q} + C^\prime p ^\prime q\left((32p^\prime)^{ p^\prime q/2} n^{-q/4})\Gamma(q/4) \right.\\
    & \left.+ (8p^\prime)^{(p^\prime + 1)q/2}n^{-q/2}\left((\log n)^{(p + 1)q/2} + \Gamma((p^\prime + 1)q/2)\right)\right),
\end{align*}
where $C^\prime$ is a universal constant and $\Gamma(\cdot)$ is the Gamma function. As a result, we have that
\begin{align*}
    & \hspace{- 3 em} \mathbb{E}\left[\left|\left(\frac{1}{n} \sum_{i=1}^n (Y_i - (X_i^\top \theta^*)^p) (X_i^\top \theta)^{p-2} X_i^\top u - \mathbb{E}[(X^\top \theta^*)^p (X^\top \theta)^{p-1} (X^\top u)^4 ] \right)\right|^q\right]\\
    & \leq p \cdot 3^{dp + 2d + q} \left(\frac{Cq}{n}\right)^{q/2} \left((4p)^{pq} + 2C^\prime pq \left(32p\right)^{pq} n^{-q/4} \Gamma(q/4) \right.\\
    & \left.+ (8p)^{(p+1)q/2}n^{-q/2}\left((\log n)^{(p+1)q/2} + \Gamma((p+1)q/2) \right)\right),
\end{align*}
for any $u\in U$. Eventually, with union bound, we obtain
\begin{align*}
    & \hspace{-1 em} \left(\mathbb{E}\left[\left\|\frac{1}{n} \sum_{i=1}^n (Y_i - (X_i^\top \theta^*)^p) (X_i^\top \theta)^{p-2} X_i X_i^\top\right\|^q\right]\right)^{1/q}\\
    & \leq 2 \cdot (17)^{d/q} \cdot 3^{\frac{dp + 2d}{q}+1} \left[\sqrt{\frac{C_p q}{n}} + \left(\frac{C_p q}{n}\right)^{3/4} + \frac{C_p}{n}(\log n + q)^{(p+1)/2}\right],
\end{align*}
where $C_p$ is a universal constant that only depends on $p$. Take $q = d(p + 3) + \log (1/\delta)$ and use the Markov inequality, we get the following bound on the second term $T_{2}$ with probability $1 - \delta$:
\begin{align}
    T_{2} \leq c_{1} \biggr(\sqrt{\frac{d + \log(1/\delta)}{n}} + \frac{1}{n} \left(d + \log \left(\frac{n}{\delta}\right) \right)^{\frac{p + 1}{2}} \biggr). \label{eq:bound_T2} 
\end{align}
\paragraph{Bounds for $T_1$ and $T_{3}$:} Using the same argument as that of $T_{2}$, we obtain the following high probability bounds for $T_{1}$ and $T_{3}$:
\begin{align}
     T_{1} & \leq c_{2} \biggr(\sqrt{\frac{d + \log(1/\delta)}{n}} + \frac{1}{n} \left(d + \log \left(\frac{n}{\delta}\right) \right)^{\frac{2p + 1}{2}} \biggr). \label{eq:bound_T1}, \\
    T_{3} & \leq c_{3} \biggr(\sqrt{\frac{d + \log(1/\delta)}{n}} + \frac{1}{n} \left(d + \log \left(\frac{n}{\delta}\right) \right)^{\frac{2p + 1}{2}} \biggr). \label{eq:bound_T3} 
\end{align}
with probability $1 - \delta$ where $c_{2}$ and $c_{3}$ are some universal constants. Plugging the bounds~\eqref{eq:bound_T2},~\eqref{eq:bound_T1}, and~\eqref{eq:bound_T3} to the bounds~\eqref{eq:key_inequality_generalized_linear_first} and~\eqref{eq:key_inequality_generalized_linear_second}, and use the condition that $n\geq C_1(d\log(d/\delta))^{2p}$ we obtain the conclusion of the lemma.
\end{proof}
\section{Proof of Gaussian Mixture Models}
\label{sec:proof:mixture_model}
In this appendix, we provide the proof for the NormGD in Gaussian mixture models.
\subsection{Homogeneous assumptions}
\label{sec:proof:low_signal_mixture_model_homogeneous}
The proof for the claim~\eqref{eq:low_signal_generalized_linear_maximum_eigenvalue} is direct from Appendix A.2.2 from~\cite{Tongzheng_2022}. Therefore, we only focus on proving the claim~\eqref{eq:low_signal_generalized_linear_minimum_eigenvalue}. Indeed, direct calculation shows that
\begin{align*}
    \nabla^2 \bar{\mathcal{L}}(\theta) = \frac{1}{\sigma^2} \parenth{I_{d} - \frac{1}{\sigma^2} \mathbb{E}\parenth{X X^{\top} \text{sech}^2 \parenth{\frac{X^{\top}\theta}{\sigma^2}}}}.
\end{align*}
We can simplify the computation of $\nabla^2 \bar{\mathcal{L}}$ via a change of coordinates. In particular, we choose an orthogonal matrix $Q$ such that $Q \theta = \|\theta\| e_1$. Here, $e_1 = (1, 0, \ldots, 0)$ is the first canonical basis in dimension $d$. We then denote $W = \frac{QX}{\sigma}$. Since $X \sim \mathcal{N}(0, I_{d})$, we have $W = (W_1, \cdots, W_d) \sim\mathcal{N}(0, I_d)$. Therefore, we can rewrite $\nabla^2 \bar{\mathcal{L}}$ as follows:
\begin{align*}
    \nabla^2 \bar{\mathcal{L}}(\theta) = \frac{1}{\sigma^2} \parenth{I_{d} - \mathbb{E}_{W}\parenth{WW^{\top} \text{sech}^2 \parenth{\frac{W_1\|\theta\|}{\sigma}}}} = \frac{1}{\sigma^2} \parenth{I_{d} - B}.
\end{align*}
It is clear that the matrix $B$ is diagonal matrix and satisfies that $B_{11} = \mathbb{E}_{W_1}\left[W_1^2 \text{sech}^2 \left(\frac{W_1 \|\theta\|}{\sigma}\right) \right]$, $B_{ii} = \mathbb{E}_{W_1}\left[\text{sech}^2 \left(\frac{W_1 \|\theta\|}{\sigma}\right) \right]$ for all $2\leq i\leq d$. An application of $\text{sech}^2(x) \leq 1 - x^2 + \frac{2}{3}x^4$ for all $x \in \mathbb{R}$ shows that
\begin{align*}
    & B_{11} \leq \mathbb{E}_{W_1}\left[W_1^2 \left(1 - \frac{W_1^2 \|\theta\|^2}{\sigma^2} + \frac{2W_1^4\|\theta\|^4}{3\sigma^4}\right)\right] = 1 - \frac{3\|\theta\|^2}{\sigma^2} + \frac{10 \|\theta\|^4}{\sigma^4},\\
    & B_{ii} \leq \mathbb{E}_{W_1}\left[\left(1 - \frac{W_1^2 \|\theta\|^2}{\sigma^2} + \frac{2W_1^4 \|\theta\|^4}{3\sigma^4}\right)\right] = 1 - \frac{\|\theta\|^2}{\sigma^2} + \frac{2\|\theta\|^4}{\sigma^4},
\end{align*}
for all $2 \leq i \leq d$. When $\|\theta\| \leq \frac{\sigma}{2}$, we have that $\frac{\|\theta\|^2}{\sigma^4}\leq \frac{1}{4}$, and hence
\begin{align*}
    & B_{11} \leq 1 - \frac{3\|\theta\|^2}{\sigma^2} + \frac{10 \|\theta\|^4}{\sigma^4}\leq 1 - \frac{\|\theta\|^2}{2\sigma^2}, \\
    & B_{ii} \leq 1 - \frac{\|\theta\|^2}{\sigma^2} + \frac{2\|\theta\|^4}{\sigma^4} \leq 1 - \frac{\|\theta\|^2}{2\sigma^2} ,\quad \forall ~ 2\leq i \leq d.
\end{align*}
Hence, as long as $\|\theta\| \leq \sigma/ 2$ we have that
\begin{align*}
    \lambda_{\min}(\nabla^2 \bar{L}(\theta)) \geq \frac{\|\theta\|^2}{2\sigma^2},
\end{align*}
which concludes the proof.
\subsection{Uniform Concentration Bounds for Mixture Models}
\label{sec:proof:uniform_concentration_mixture_model_homogeneous}
See Corollary 4 in~\cite{Siva_2017} for the proof of the uniform concentration result between $\nabla \bar{\mathcal{L}}_{n}(\theta)$ and $\nabla \bar{\mathcal{L}}(\theta)$ in equation~\eqref{eq:concentration_Hessian_mixture_model} for the strong signal-to-noise regime. Now, we prove the uniform concentration bounds between $\nabla^2 \bar{\mathcal{L}}_{n}(\theta)$ and $\nabla \bar{\mathcal{L}}(\theta)$ in equations~\eqref{eq:concentration_Hessian_mixture_model} and~\eqref{eq:concentration_Hessian_mixture_model_low_signal} for both the strong signal-to-noise and low signal-to-noise regimes. It is sufficient to prove the following lemma.
\begin{lemma}
\label{lemma:unified_concentration_bound_mixture_model}
There exist universal constants $C_{1}$ and $C_{2}$ such that as long as $n \geq C_{1} d \log(1/ \delta)$ we obtain that
\begin{align}
    \sup_{\theta \in \mathbb{B}(\theta^{*},r)} \|\nabla^2 \mathcal{L}_n(\theta) - \nabla^2 \mathcal{L}(\theta)\|_{\text{op}} \leq C_{2} (\|\theta^{*}\| + \sigma^2) \sqrt{\frac{d + \log(1/\delta)}{n}}. \label{eq:unified_concentration_bound_mixture_model}
\end{align}
\end{lemma}
\begin{proof}[Proof of Lemma~\ref{lemma:unified_concentration_bound_mixture_model}]For the sample log-likelihood function of the Gaussian mixture model, direct calculation shows that
\begin{align*}
    \bar{\mathcal{L}}_{n}(\theta) = \frac{\|\theta\|^2 + \frac{1}{n}\sum_{i=1}^n \|X_{i}\|^2}{2\sigma^2} - \frac{1}{n} \sum_{i = 1}^{n} \log \parenth{\exp \parenth{-\frac{X_{i}^{\top}\theta}{\sigma^2}} + \exp \parenth{\frac{X_{i}^{\top}\theta}{\sigma^2}}} - \log(2 (\sqrt{2 \pi})^{d} \sigma^{d}).
\end{align*}
Therefore, we find that
\begin{align*}
    \nabla \bar{\mathcal{L}}_{n}(\theta) = \frac{\theta}{\sigma^2} - \frac{1}{n \sigma^2} \sum_{i = 1}^{n} X_{i} \tanh \parenth{\frac{X_{i}^{\top} \theta}{\sigma^2}}, \\
    \nabla^2 \bar{\mathcal{L}}_{n}(\theta) = \frac{1}{\sigma^2} \parenth{I_{d} - \frac{1}{n\sigma^2} \sum_{i = 1}^{n} X_{i} X_{i}^{\top} \text{sech}^2 \parenth{\frac{X_{i}^{\top}\theta}{\sigma^2}}},
\end{align*}
where $\text{sech}^2(x) = \frac{4}{(\exp(-x) + \exp(x))^2}$ for all $x \in \mathbb{R}$. 

For the population log-likelihood function, we have
\begin{align*}
    \nabla^2 \bar{\mathcal{L}}(\theta) = \frac{1}{\sigma^2} \parenth{I_{d} - \frac{1}{\sigma^2} \mathbb{E}\parenth{X X^{\top} \text{sech}^2 \parenth{\frac{X^{\top}\theta}{\sigma^2}}}}.
\end{align*}
Therefore, we obtain that
\begin{align*}
    \nabla^2 \bar{\mathcal{L}}_n(\theta) - \nabla^2 \bar{\mathcal{L}}(\theta) = \frac{1}{\sigma^4}\left(\frac{1}{n}\sum_{i=1}^n X_i X_i^\top \text{sech}^2\parenth{\frac{X_{i}^{\top}\theta}{\sigma^2}} -\mathbb{E} \parenth{X X^{\top} \text{sech}^2 \parenth{\frac{X^{\top}\theta}{\sigma^2}}} \right).
\end{align*}
Use the variational characterization of operator norm, it's sufficient to consider
\begin{align*}
    T = & \sup_{u \in \mathbb{S}^{d-1}, \theta\in\mathbb{B}(\theta^*, r)} \left|\frac{1}{n}\sum_{i=1}^n(X_i^\top u)^2 \text{sech}^2\left(\frac{X_i^\top \theta}{\sigma^2}\right) - \mathbb{E}\left((X^\top u)^2 \text{sech}^2\left(\frac{X^\top \theta}{\sigma^2}\right)\right)\right|.
\end{align*}
With a standard discretization argument (e.g. Chapter 6 in \citep{Wainwright_nonasymptotic}), let $U$ be a $1/8$-cover of $\mathbb{S}^{d-1}$ under $\|\cdot\|_2$ whose cardinality can be upper bounded by $17^d$, we have that
\begin{align*}
     & \hspace{-5 em} \sup_{u \in \mathbb{S}^{d-1}, \theta\in \mathbb{B}(\theta^*, r)} \left|\frac{1}{n}\sum_{i=1}^n(X_i^\top u)^2\text{sech}^2\left(\frac{X_i^\top \theta}{\sigma^2}\right)  - \mathbb{E}\left((X^\top u)^2 \text{sech}^2\left(\frac{X^\top \theta}{\sigma^2}\right)\right)\right|\\
     \leq & 2 \sup_{u\in U, \theta\in\mathbb{B}(\theta^*, r)}\left|\frac{1}{n}\sum_{i=1}^n(X_i^\top u)^2\text{sech}^2\left(\frac{X_i^\top \theta}{\sigma^2}\right) - \mathbb{E}\left((X^\top u)^2\text{sech}^2\left(\frac{X^\top \theta}{\sigma^2}\right)\right)\right|.
\end{align*}
With a symmetrization argument, we have that
\begin{align*}
    & \hspace{- 5 em} \mathbb{E}\left[ \sup_{\theta\in\mathbb{B}(\theta^*, r)}\left|\frac{1}{n}\sum_{i=1}^n(X_i^\top u)^2\text{sech}^2\left(\frac{X_i^\top \theta}{\sigma^2}\right)  - \mathbb{E}\left((X^\top u)^2 \text{sech}^2\left(\frac{X^\top \theta}{\sigma^2}\right)\right)\right|\right]\\
    \leq & \mathbb{E}\left[\sup_{\theta\in\mathbb{B}(\theta^*, r)} \left|\frac{2}{n}\sum_{i=1}^n \varepsilon_i(X_i^\top u)^2\text{sech}^2\left(\frac{X_i^\top \theta}{\sigma^2}\right)\right|\right],
\end{align*}
where $\{\varepsilon_i\}$ is a set of i.i.d Rademacher random variable. Define
\begin{align*}
    Z:= \sup_{\theta\in\mathbb{B}\theta^*, r}\left|\frac{2}{n}\sum_{i=1}^n \varepsilon_i(X_i^\top u)^2\text{sech}^2\left(\frac{X_i^\top \theta}{\sigma^2}\right)\right|.
\end{align*}
For $x\in\mathbb{R}$, define $(x)_+ = \max(x, 0)$, $(x)_{-} = \min(x, 0)$. Furthermore, for random variable $X$, we denote
\begin{align*}
    \|X\|_q = \left(\mathbb{E}[X^q]\right)^{1/q}.
\end{align*}
With Theorem 15.5 in \citep{boucheron2013concentration}, there exists an absolute constant $C$, such that for all $q\geq 2$,
\begin{align*}
    \left\|\left(Z - \mathbb{E}[Z]\right)_{+}\right\|_q \leq \sqrt{C q \|V^{+}\|_{q/2}},
\end{align*}
where, with the proof idea of Theorem 15.14 in \citep{boucheron2013concentration}, $V^+$ can be bounded as
\begin{align*}
    V^+ \leq & \sup_{\theta\in \mathbb{B}(\theta^*, r)}  \frac{1}{n}\mathbb{E} \brackets{(X^\top u)^4 \text{sech}^4\left(\frac{X^\top \theta}{\sigma^2}\right)} + \sup_{\theta\in\mathbb{B}(\theta^*, r)} \frac{4}{n^2}\sum_{i=1}^n (X_i^\top u)^4 \text{sech}^4\left(\frac{X_i^\top \theta}{\sigma^2}\right)\\
    \leq & \frac{1}{n}\mathbb{E}\brackets{(X^\top u)^4} + \frac{4}{n^2}\sum_{i=1}^n (X_i^\top u)^4.
\end{align*}
Here we use the fact that $0\leq \text{sech}^2(x) \leq 1$ for all $x$ in the last step.
Notice that, $X_i^\top u \sim \frac{1}{2}\mathcal{N}(u^\top \theta^*, \sigma^2 ) + \frac{1}{2}\mathcal{N}(-u^\top \theta^*, \sigma^2)$. We can verify that there exists absolute constant $c$, such that
\begin{align*}
    \mathbb{E}\left[(X_i^\top u)^{2p}\right] = \mathbb{E}_{X\sim \mathcal{N}(u^\top \theta^*, \sigma^2)}\left[(X^\top u)^{2p}\right]\leq (2cp)^p(\|\theta^*\|^2 + \sigma^2)^{p}.
\end{align*}
Apply Lemma 2 in \citep{mou2019diffusion} with $Y_{i} = (X_i^\top u)^4$, we have that
\begin{align*}
    & \frac{1}{n}\sum_{i=1}^n(X_i^\top u)^4 \leq c (\|\theta^*\|^2 + \sigma^2)^2 \left( 1 + \sqrt{\frac{\log 1/\delta}{n}} + \frac{\sqrt{\log n/\delta}\log 1/\delta}{n}\right),
\end{align*}
with probability at least $1-\delta$ for some absolute constant $c$. As $n \geq C_1(d + \log 1/\delta)$, we can conclude that
\begin{align*}
    V^+ \leq \frac{c'(\|\theta^*\|^2 + \sigma^2)^2}{n}
\end{align*}
for some universal constant $c'$. Furthermore, as $Z \geq 0$, we have that
\begin{align*}
    \left\|\left(Z - \mathbb{E}[Z]\right)_{-}\right\|_q \leq \mathbb{E}[Z].
\end{align*}
Hence, with Minkowski's inequality, we have that
\begin{align*}
    \|Z\|_{q} \leq 2 \mathbb{E}[Z] + \sqrt{\frac{cq (\|\theta^*\|^2 + \sigma^2)^2}{n}},
\end{align*}
for some absolute constant $c$. We now bound $\mathbb{E}[Z]$. Consider the following function class
\begin{align*}
    \mathcal{G}:=\left\{g_{\theta} : X \to (X^\top u)^2 \text{sech}^2\left(\frac{X^\top \theta}{\sigma^2}\right)\bigg| \theta\in \mathbb{R}^d\right\}.
\end{align*}
Clearly, the function class $\mathcal{G}$ has an envelop function $\bar{G}(X) =(X^\top u)^2$. Meanwhile, as the function $\text{sech}^2$ is monotonic in $(-\infty, 0)$ and $(0, \infty)$ and $\theta$ here effects only in the form $X^\top \theta$. Following some algebra we know the VC subgraph dimension of $\mathcal{G}$ is at most $d + 2$. Hence, the $L_2$-covering number of $\mathcal{G}$ can be bounded by
\begin{align*}
    \bar{N}(t) := \sup_{Q}\left|\mathcal{N}(\mathcal{G}, \|\cdot\|_{L_2(Q)}, t \|\bar{G}\|_{L_2(Q)})\right|\leq (1/t)^{c(d+1)}
\end{align*}
for any $t > 0$, where $c$ is an absolute constant. With Dudley's entropy integral bound (e.g. \citep[Theorem 5.22]{Wainwright_nonasymptotic}), we have
\begin{align*}
    \mathbb{E}[Z] \leq & c \sqrt{\frac{\sum_{i=1}^n (X_i^\top u)^4}{n^2}} \int_{0}^1 \sqrt{1 + \bar{N}(t)} dt\\
\leq & c \sqrt{\frac{d (\|\theta^*\|^2 + \sigma^2)^2}{n}}
\end{align*}
for some absolute constant $c$.

Take $q = \log 1/\delta$, use the Markov equality and an union bound over $u$, we know there exist universal constants $C_1$ and $C_2$, such that the following inequality
\begin{align*}
    \sup_{\theta\in\mathbb{B}(\theta^*, r)}\left|\frac{1}{n}\sum_{i=1}^n X_i X_i^\top\text{sech}^2\left(\frac{X_i^\top \theta}{\sigma^2}\right)  - \mathbb{E}\left(X_i X_i^\top \text{sech}^2\left(\frac{X^\top \theta}{\sigma^2}\right)\right)\right| \leq C_2(\|\theta^*\|^2 + \sigma^2) \sqrt{\frac{d + \log 1/\delta}{n}}
\end{align*}
holds with probability at least $1-\delta$ as long as $n \geq C_1(d + \log 1/\delta)$. As a consequence, we obtain the conclusion of the lemma.

\bibliographystyle{abbrv}
\bibliography{Nhat}
\end{proof}
\end{document}

%% file: final_macros.tex
\newcommand{\brackets}[1]{\left[ #1 \right]}
\newcommand{\parenth}[1]{\left( #1 \right)}

\newtheoremstyle{named}{}{}{\itshape}{}{\bfseries}{.}{.5em}{\thmnote{#3's }#1}
\theoremstyle{named}

\theoremstyle{plain}

\newtheorem{theorem}{Theorem}
\newtheorem{proposition}{Proposition}
\newtheorem{lemma}{Lemma}

\newtheorem{corollary}{Corollary}

\newlength{\widebarargwidth}
\newlength{\widebarargheight}
\newlength{\widebarargdepth}

\makeatletter
\long\def\@makecaption#1#2{
        \vskip 0.8ex
        \setbox\@tempboxa\hbox{\small {\bf #1:} #2}
        \parindent 1.5em  
        \dimen0=\hsize
        \advance\dimen0 by -3em
        \ifdim \wd\@tempboxa >\dimen0
                \hbox to \hsize{
                        \parindent 0em
                        \hfil
                        \parbox{\dimen0}{\def\baselinestretch{0.96}\small
                                {\bf #1.} #2
                                }
                        \hfil}
        \else \hbox to \hsize{\hfil \box\@tempboxa \hfil}
        \fi
        }
\makeatother

\long\def\comment#1{}
\definecolor{battleshipgrey}{rgb}{0.52, 0.52, 0.51}
\definecolor{darkgray}{rgb}{0.66, 0.66, 0.66}
\definecolor{darkgreen}{rgb}{0.0, 0.2, 0.13}
\definecolor{darkspringgreen}{rgb}{0.09, 0.45, 0.27}
\definecolor{dukeblue}{rgb}{0.0, 0.0, 0.61}
\definecolor{olivedrab7}{rgb}{0.24, 0.2, 0.12}
\definecolor{darkblue}{rgb}{0.0, 0.0, 0.55}
\definecolor{darkscarlet}{rgb}{0.34, 0.01, 0.1}
\definecolor{candyapplered}{rgb}{1.0, 0.03, 0.0}
\definecolor{ao(english)}{rgb}{0.0, 0.5, 0.0}
\definecolor{applegreen}{rgb}{0.55, 0.71, 0.0}
